\documentclass[onefignum,onetabnum]{siamonline220329}

\usepackage{lipsum}
\usepackage{amsfonts}
\usepackage{bbm}
\usepackage{graphicx}
\usepackage{bm}
\usepackage{epstopdf}
\usepackage{algorithmic}
\ifpdf
  \DeclareGraphicsExtensions{.eps,.pdf,.png,.jpg}
\else
  \DeclareGraphicsExtensions{.eps}
\fi

\usepackage{enumitem}
\setlist[enumerate]{leftmargin=.5in}
\setlist[itemize]{leftmargin=.5in}

\newsiamremark{remark}{Remark}
\newsiamremark{hypothesis}{Hypothesis}
\crefname{hypothesis}{Hypothesis}{Hypotheses}
\newsiamthm{claim}{Claim}

\headers{Convergence of Recurrent Neural Networks}{S. Cayci, A. Eryilmaz}

\title{Convergence of Gradient Descent for Recurrent Neural Networks: \\A Nonasymptotic Analysis}

\author{Semih Cayci\thanks{Department of Mathematics, RWTH Aachen University, Aachen, Germany
  (\email{cayci@mathc.rwth-aachen.de}).}
\and Atilla Eryilmaz\thanks{Department of Electrical and Computer Engineering, The Ohio State University, Columbus, OH 
  (\email{eryilmaz.2@osu.edu}).}}

\usepackage{amsopn}

\usepackage{blindtext}
\usepackage{titlesec}

\usepackage{cleveref}
\usepackage{mathrsfs}

\newcommand{\ms}{\mu_{0,t}}
\newcommand{\mb}{\mu_{1,t}}
\newcommand{\Sf}[1]{\mathsf{#1}}
\newcommand{\pact}[1]{\Gamma_{#1}^\uind{i}}
\newcommand{\qact}[1]{\check{\Gamma}_{#1}^\uind{i}}
\newcommand{\Zt}{Z_t^\uind{i}(\X;\Phi(0))}
\newcommand{\vpt}{\Gamma_t^\uind{i}}
\newcommand{\vvpt}{\check{\Gamma}_t^\uind{i}}

\newcommand{\exit}{\mathsf{exit}}
\newcommand{\bd}{\varphi_T^\mathsf{max}}
\newcommand{\norm}[2]{\|\cdot\|_{#1,#2}}
\newcommand{\tF}{\tilde{F}}
\newcommand{\iwl}{\cF_{\bm{\nb}}}
\newcommand{\er}{\widehat{\mathcal{R}}_\cS}
\newcommand{\LL}{\left}
\newcommand{\RR}{\right}
\newcommand{\W}{\mathbf{W}}
\newcommand{\U}{\mathbf{U}}

\newcommand{\Lin}{F_t^\mathsf{Lin}}
\newcommand{\nb}{\bar{\nu}}
\newcommand{\del}{\partial}
\newcommand{\ind}{\mathbbm{1}}

\newcommand{\bE}{\mathbb{E}}

\newcommand{\bN}{\mathbb{N}}

\newcommand{\bP}{\mathbb{P}}

\newcommand{\bR}{\mathbb{R}}

\newcommand{\X}{\mathbf{X}}

\newcommand{\bZ}{\mathbb{Z}}
\newcommand{\Xt}{{[\X]_t}}
\newcommand{\calF}{\mathcal{F}}
\newcommand{\fK}{\kappa}

\newcommand{\cB}{\mathscr{B}}
\newcommand{\cC}{\mathscr{C}}

\newcommand{\cF}{\mathscr{F}}

\newcommand{\cH}{\mathscr{H}}

\newcommand{\cL}{\mathscr{L}}

\newcommand{\cO}{\mathcal{O}}

\newcommand{\cS}{\mathcal{S}}

\newcommand{\bv}{\bm{v}}

\newcommand{\gz}{{>0}}
\newcommand{\Wii}{{W_{ii}}}

\newcommand{\uind}[1]{{({#1})}}

\newtheorem*{theoreminf}{Theorem}
\newtheorem*{corollaryinf}{Corollary}

\ifpdf
\hypersetup{
  pdftitle={An Example Article},
  pdfauthor={D. Doe, P. T. Frank, and J. E. Smith}
}
\fi

\begin{document}

\maketitle

\begin{abstract}
We analyze recurrent neural networks with {diagonal hidden-to-hidden weight matrices}, trained with gradient descent in the supervised learning setting, and prove that gradient descent can achieve optimality \emph{without} massive overparameterization. Our in-depth nonasymptotic analysis (i) provides {improved} bounds on the network size $m$ in terms of the sequence length $T$, sample size $n$ and ambient dimension $d$, and (ii) identifies the significant impact of long-term dependencies in the dynamical system on the convergence and network width bounds characterized by a cutoff point that depends on the Lipschitz continuity of the activation function. Remarkably, this analysis reveals that an appropriately-initialized recurrent neural network trained with $n$ samples can achieve optimality with a network size $m$ that scales only logarithmically with $n$. This sharply contrasts with the prior works that require high-order polynomial dependency of $m$ on $n$ to establish strong regularity conditions. Our results are based on an explicit characterization of the class of dynamical systems that can be approximated and learned by recurrent neural networks via norm-constrained transportation mappings, and establishing local smoothness properties of the hidden state with respect to the learnable parameters.
\end{abstract}



\section{Introduction}
Recurrent neural networks (RNNs) have become one of the most widely-used approaches in learning dynamical systems that inherit memory, which have been widely used across a myriad of applications including natural language processing, speech recognition, time-series analysis, text generation, visual attention \cite{rumelhart1986learning, goodfellow2016deep, pascanu2013construct, graves2005framewise, mnih2014recurrent, sutskever2011generating, graves2013speech, graves2013generating}. In particular, the possibility of deploying large recurrent neural networks enabled their widespread use in practice by resolving the stability and of gradient-descent-based training of RNNs \cite{goodfellow2016deep, jaeger2002adaptive, pascanu2013construct}. Despite the wide popularity and success of RNNs in practice, a concrete theoretical understanding of RNNs is still in a nascent stage. In particular, the ability of RNNs in efficient learning of dynamical systems, and the associated challenges such as the so-called exploding gradient problem and the required number of neurons for provably-good performance necessitate theoretical characterization. Recently, there has been a surge of interest in the analysis of generalization properties \cite{allen2019can, boutaib2022path} and the kernel limit of recurrent neural networks \cite{alemohammad2020recurrent}. A number of theoretical studies \cite{allen2019can, allen2019convergence} have recently investigated the optimization and generalization performance of RNNs trained with gradient descent in the kernel regime, where the neural network width $m$ scales polynomially with the training set size $n$ and sequence-length $T$ with extremely large exponents as large as 100. Since contemporary neural networks typically employ much more modest number of hidden units, such large dependencies of the network width $m$ on $n, T$ and {error probability} $\delta$ create a significant gap between theory and practice. In this work, we address this problem with the ultimate goal of narrowing the gap between theory and practice for the standard RNNs.

We ask the following questions in this work:
\begin{itemize}
    \item Can recurrent neural networks trained with gradient descent achieve optimality with number of neurons that grows logarithmically with $n$, $1/\delta$ and polynomially with $T$? 
    \item  What is the impact of memory (long-term dependencies) in the dynamical system on the network width and convergence rate?
    \item What class of dynamical systems can be expressed by the reproducing kernel Hilbert space of the neural tangent kernel? How efficiently such neural networks can be approximated and learned by finite-width neural networks?
    \item What is the impact of regularization on the iteration complexity and scalability of training RNNs?
    
\end{itemize}
To answer these questions, we perform a finite-time and finite-width analysis of RNNs and establish sharp convergence bounds.

\textbf{Neural tangent kernel analysis.} An important paradigm on the analysis of overparameterized neural networks is the so-called neural tangent kernel (NTK) analysis, which considers wide randomly-initialized neural networks in the near-initialization regime \cite{jacot2018neural, du2019gradient, chizat2019lazy}. The original works required massive overparameterization, which implies the number of neurons $m$ scales polynomially with the sample size $n$ \cite{du2019gradient, chizat2019lazy, oymak2020toward}. An alternative approach was proposed by \cite{ji2019polylogarithmic}, proving that shallow feedforward ReLU networks can achieve global optimality with only $\log(n)$ width for binary classification under early stopping. Our analysis based on early stopping builds on, and utilizes tools from \cite{ji2019polylogarithmic}. We note that these works consider only feedforward neural networks, thus they are not applicable for RNNs, which are significantly more challenging to analyze due to the weight-sharing and the existence of memory.

\textbf{Recurrent neural networks.} RNNs have been shown to be universal approximators for dynamical systems, which prove their effectiveness in the learning context \cite{schafer2007recurrent, gonon2023approximation, bishop2023recurrent}. On the other hand, these works on the approximation properties of RNNs are existence results, which do not prescribe methods to find the optimal weights for a given learning problem. As such, empirical risk minimization paradigm is predominantly used in supervised learning for dynamical systems as well, and training RNNs by using descent-type optimization algorithms has been the standard method due to their impressive empirical performance \cite{graves2012supervised, robinson1987utility, goodfellow2016deep}. On the other hand, a concrete theoretical understanding of the training dynamics of RNNs with descent-type methods remains elusive due to the specific challenges of RNNs, including weight-sharing and long-term dependencies. A convergence analysis for linear RNNs with identity function as the activation function was given in \cite{hardt2018gradient}. In order to extend the NTK analysis to RNN architectures with nonlinear activation functions, the existence of the neural tangent kernel was proved for RNNs in \cite{alemohammad2020recurrent}. In \cite{allen2019convergence}, the convergence of GD for RNNs was analyzed in the near-initialization regime, but the required number of neurons was polynomial in $n$ with exponents as big as 100, leading to a massive overparameterization bounds that are far from practical observations. In \cite{lam2023kernel}, the infinite-width limit of RNNs was analyzed in the kernel regime in the online setting with ergodic input data. Unlike these existing works, in our work, we establish a finite-time, finite-sample and finite-width analysis of gradient descent for recurrent neural networks to prove sharp bounds on $m$ in terms of system parameters including the number of samples $n$ and the sequence-length $T$. 

\subsection{Our Contributions}
The goal in our paper is to establish sharp bounds on the required number of neurons $m$ and the number of iterations $\tau$ of gradient descent for \emph{standard} Elman-type RNNs to achieve near-optimality in learning dynamical systems on a sequence of length $T$. In particular, we have the following main contributions.
\begin{itemize}
    \item \textbf{Sharp bounds for RNNs trained with gradient descent.} 
    We show that, a randomly-initialized RNN of
    \begin{itemize}
        \item[$\star$] $m=\Omega\LL(\Sf{poly}(\mu_T, T)\frac{\log(nT/\delta)}{\epsilon^2}\RR)$ neurons trained with $\tau = \Omega(\Sf{poly}(\mu_T)/\epsilon^2)$ iterations of \textit{projected} gradient descent {and stochastic gradient descent}, and
        \item[$\star$] $m=\Omega(\Sf{poly}(\mu_T, T)\log(nT/\delta)/\epsilon^3)$ neurons trained with $\tau = \Omega(\Sf{poly}(\mu_T,T)/\epsilon^3)$ iterations of gradient descent (without projection),
    \end{itemize} achieves $\epsilon>0$ training error with probability at least $1-\delta$ over the random initialization. Notably, the required neural network size $m$ depends only logarithmically on the number of samples $n$.
    
    \item \textbf{Impact of memory on the required network size and iteration complexity.} Our results reflect the impact of long-term dependencies (i.e., memory) in the dynamical system and the resulting exploding gradient problem in supervised learning for dynamical systems. The quantity $\mu_T$ that scale both $m$ and $\tau$ is $\cO(1)$ for dynamical systems with short-term memory, while it becomes $e^{\Omega(T)}$ for systems with long-term memory, which is characterized by a cutoff that depends on the regularity of the activation function.
    \item \textbf{Characterization of dynamical systems in the infinite-width limit.} We provide an explicit characterization of the class of dynamical systems that can be characterized by the infinite-size limit of randomly-initialized RNNs in the kernel regime.
    \item \textbf{Impact of explicit regularization on RNNs.} We show the impact of explicit regularization on the performance of RNNs for learning dynamical systems in terms of the required network size and the number of iterations. In particular, we prove that $\norm{2}{\infty}$-regularization improves both the iteration complexity and the required network width by a factor of $\cO(1/\epsilon)$ in learning the same class of dynamical systems.
\end{itemize}
We propose new proof techniques based on the local smoothness properties of the hidden state for RNNs, which enable a sharp analysis by addressing the unique complications of RNNs in the near-initialization regime.




\subsection{Notation}
For any $\textbf{X}=[X_1,X_2,\ldots,X_n]\in\bR^{m\times n}$ and $t\leq n$, we denote $[\textbf{X}]_t=[X_0,X_1,\ldots,X_{t-1}]\in\bR^{m\times t}$. For any $x\in\bR^n$, $\cB_{p}(x,\rho)=\LL\{y\in\bR^{n}:\|x-y\|_{p}\leq \rho\RR\}.$ For a finite set $A$, $\mathsf{Unif}(A)$ denotes uniform distribution over $A$, and $\mathsf{Rad}_\alpha$ denotes $\Sf{Unif}(\{-\alpha,\alpha\})$ for $\alpha\in\bR_+$. $\varsigma(X_j,j\in\mathcal{J})$ denotes the $\sigma$-algebra generated by a collection of random variables $X_j,j\in\mathcal{J}$. A differentiable function $f:\bR^d\rightarrow\bR$ is $\beta$-smooth if it has $\beta$-Lipschitz gradients.

\section{Learning Dynamical Systems with Empirical Risk Minimization}
In this section, we formally define the supervised learning setting for dynamical systems, describe the RNN architecture that we will study, and provide the algorithms to solve the empirical risk minimization.
\subsection{Empirical Risk Minimization for Dynamical Systems}
We consider a causal discrete-time dynamical system that evolves according to 
\begin{align*}
    h_t&=g\LL(h_{t-1},X_{t-1}\RR),\\
    Y_t&=\phi_t(h_t),
\end{align*}
for $t=1,2,\ldots,T$ with the initial condition $h_0\in\bR$, where we have an $\bR^{d\times T}\times\bR^T$-valued random variable $({\mathbf{X}},Y)$ with the input matrix ${\mathbf{X}}=[X_0,X_1,\ldots,X_{T-1}]$ with distribution $P$, and the corresponding output vector $Y=[Y_1,Y_2,\ldots,Y_T]^\top\in\bR^T$. We denote the input-output relation more compactly as  $$Y_t=F_t^\star([{\mathbf{X}}]_t),~t=1,2,\ldots,T.$$ We denote $\bm{F}^\star = (F_t^\star)_{t>0}$. Given a class of predictors $F_t(\cdot;\Phi),t=1,2,\ldots,T$ parameterized by $\Phi\in\bR^p$, our goal is to find $\Phi^\star\in\bR^p$ that minimizes the empirical risk 
\begin{equation}
\er(\Phi)=\frac{1}{n}\sum_{j=1}^n\sum_{t=1}^T\Big(F_t([\mathbf{X}^{(j)}]_t;\Phi)-Y_t^{(j)}\Big)^2,
\label{eqn:erm}
\end{equation}
where $\cS=\Big((\mathbf{X}^{(1)},Y^{(1)}),\ldots,(\mathbf{X}^{(n)},Y^{(n)})\Big)$ is the training set with $n$ independent and identically distributed (iid) input-output pairs $(\mathbf{X}^{(j)},Y^{(j)})$ such that $\mathbf{X}^\uind{j}\sim P$ and $Y_t^\uind{j}=F_t^\star([\mathbf{X}^\uind{j}]_t)$ for $j=1,2,\ldots,n$ and $t=1,2,\ldots,T$. Throughout the paper, we assume that $P$ is compactly supported such that $\max_{0\leq t < T}\|X_t\|_2\leq 1$ almost surely. We usually drop the notation $[\X]_t$ and write $F_t^\star([\X]_t)$ as $F_t^\star(\X)$ and $F_t(\Xt;\Phi)$ as $F_t(\X;\Phi)$ instead.

\subsection{Elman-Type Recurrent Neural Networks}
The input data is denoted as the matrix $\mathbf{X}=[X_0, X_1,\ldots]$ where $X_t\in\bR^d$ for the ambient dimension $d\in\bN$. We consider an Elman-type recurrent neural network (RNN) of width $m\in\bN$ with learnable parameters $\W\in\bR^{m\times m}$ and $\U\in\bR^{m\times d}$, where the rows of $\U$ are denoted as $U_i^\top$ for $i=1,2,\ldots,m$. 

In this work, we consider \emph{diagonal} {hidden-to-hidden weight matrix} $\W\in\bR^m\times\bR^m$ and general {input-to-hidden weight matrix} $\U\in\bR^{m\times d}$, which facilitates the analysis considerably in the large-network setting, while still capturing the essence of RNNs. We note that diagonal weight matrices are widely considered in the literature to understand the dynamics of deep neural networks \cite{woodworth2020kernel, chou2023robust, nacson2022implicit}. {We denote $\Theta_i := \begin{bmatrix}
    \Wii\\U_i
\end{bmatrix}\in\bR^{d+1}$ for $i=1,2,\ldots,m$, and $\Theta^\top := (\Theta_1^\top~\Theta_2^\top~\ldots~\Theta_m^\top)$.}

The central structure in an RNN is the sequence of hidden states $H_t\in\bR^m$, which evolves according to the following rule:
\begin{equation}
    H_t^{(i)}(\X;\Theta)=\sigma\Big(W_{ii}H_{t-1}^{(i)}(\mathbf{X};\Theta)+U_i^\top X_{t-1}\Big),~i=1,2,\ldots,m,
\end{equation}
for $t=1,2,\ldots$ where $H_0=0$, and $\sigma:\bR\rightarrow\bR$ is the (nonlinear) activation function. The evolution of the hidden state $(H_t)_{t>0}$ over time is illustrated in Figure \ref{fig:rnn}.
\begin{figure}[hbt]
    \centering
    \includegraphics[width=.7\textwidth]{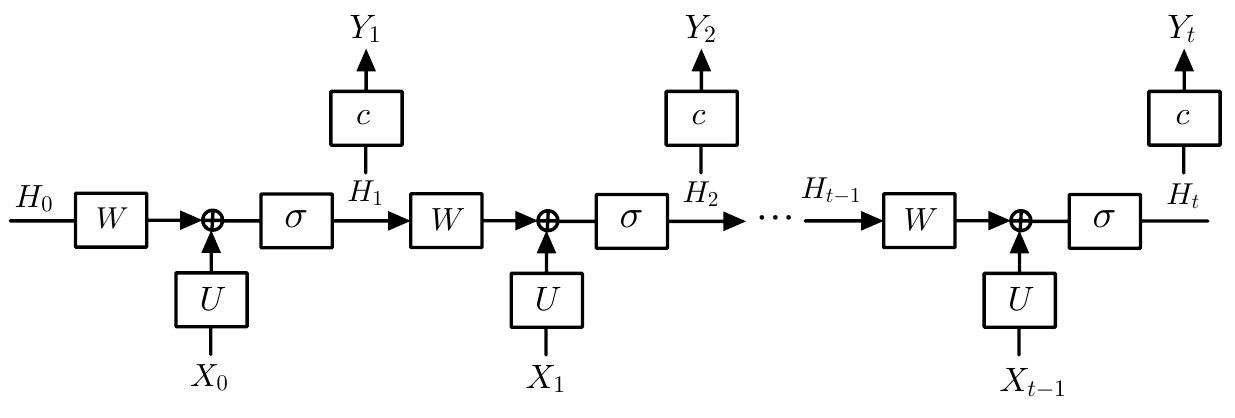}
    \caption{Unfolded representation of an Elman-type recurrent neural network in the matrix notation $Y_t = c^\top\vec{\sigma}(\W H_{t-1}+\U X_{t-1})$, where $\vec{\sigma}:\bR^m\mapsto\bR^m$ applies $\sigma$ pointwise to each component of its input.}
    \label{fig:rnn}
\end{figure}
We consider a linear read-out layer with weights $c\in\bR^m$, which leads to the output
\begin{align}
    \begin{aligned}
    F_t(\mathbf{X};\Phi) &= \frac{1}{\sqrt{m}}\sum_{i=1}^mc_iH_t^{(i)}(\X;\Theta),\\
    &= \frac{1}{\sqrt{m}}\sum_{i=1}^mc_i\sigma\Big(W_{ii}H_{t-1}^{(i)}(\X;\Theta)+U_i^\top X_{t-1}\Big),
    \label{eqn:rnn}
    \end{aligned}
\end{align}
{where $\Phi := \begin{bmatrix}
    \Theta\\c
\end{bmatrix}$ is the parameter set of the recurrent neural network.}
The characteristic property of RNNs is weight-sharing: throughout all time-steps $t\in\bZ_+$, the same weights $\Phi$ are utilized, which enables the hidden state $(H_t)_{t
\in\bZ_+}$ to summarize the entire history $[\mathbf{X}]_t$. Also, note that the neural network $F_t$ reduces to a two-layer feedforward neural network of width $m$ if $\W=0$, therefore the distinctive property of RNNs is the history-dependence established through $\W\neq 0$ and weight-sharing.

In this work, we consider a smooth activation function $\sigma\in\cC^2(\bR,\bR)$ such that
\begin{align}
    \sup_{z\in\bR}|\sigma(z)| \leq \sigma_0,\mbox{~}
    \sup_{z\in\bR}|\sigma'(z)| \leq \sigma_1,\mbox{~and~}
    \sup_{z\in\bR}|\sigma''(z)| \leq \sigma_2.
    \label{eqn:sigma}
\end{align}
Many widely-used activations including $\tanh$, logistic, Gaussian have these properties \cite{goodfellow2016deep}.

\textbf{Notation.} For any $g(\Phi)=g(\Theta,c)$, we use the notation 
\begin{align*}
    \partial_{\Theta_i}g(\Phi) &= \begin{bmatrix}
        \partial_{W_{ii}}g(\Phi)\\
        \del_{U_i}g(\Phi)
    \end{bmatrix},~\partial_\Theta g(\Phi) = \begin{bmatrix}
        \partial_{\Theta_1}g(\Phi)\\
        \vdots\\
        \partial_{\Theta_m}g(\Phi)
    \end{bmatrix},\del_c g(\Phi)=\begin{bmatrix}\del_{c_1}g(\Phi)\\\vdots\\\del_{c_m}g(\Phi)
    \end{bmatrix}.
\end{align*}

\subsection{Gradient Descent for Recurrent Neural Networks}
The standard method to solve empirical risk minimization problem \eqref{eqn:erm} is to use gradient descent. In this section, we present the basic gradient descent algorithm for Elman-type RNNs and a projected gradient descent that incorporates projection as a form of regularization in the near-initialization regime.

One key concept is random initialization, which is widely used in practice \cite{goodfellow2016deep} and yields the basis of the kernel analysis \cite{jacot2018neural, chizat2019lazy, allen2019convergence}. In this paper, we assume that $m$ is even without any loss of generality, and use the following symmetric initialization \cite{chizat2019lazy}. 
\begin{definition}[Symmetric random initialization]
Let ${c_i(0)}\sim\Sf{Rad}_1,V_i\sim\Sf{Rad}_\alpha,U_i(0)\sim\mathcal{N}(0,I_d)$ independently for all $i\in\{1,2,\ldots,m/2\}$ and independently from each other, and  for $i\in\{m/2+1,\ldots,m\}$, ${c_{i}(0)=-c_{i-m/2}(0)},V_i=V_{i-m/2}$ and $U_i(0)=U_{i-m/2}(0)$. {Then, we call $\Phi(0)=\big(\Phi_1^\top(0)\ldots\Phi_m^\top(0)\big)^\top$ a symmetric random initialization where $\Phi_i(0)=\begin{bmatrix}
    \Theta_i(0)\\
    c_i(0)
\end{bmatrix}$, $\Theta_i(0)=\begin{bmatrix}
    V_i\\U_i(0)
\end{bmatrix}$ for $i=1,2,\ldots,m$, and $\Theta^\top(0)=\LL(\Theta_1^\top(0),\Theta_2^\top(0),\ldots,\Theta_m^\top(0)\RR).$}

    \label{def:sym-initialization}
\end{definition} 
\noindent The symmetrization ensures that $F_t([\textbf{X}]_t;\Phi(0))=0$ for any $t\geq0$ and input $\mathbf{X}\in\bR^{d\times T}$.

\begin{itemize}
    \item \textbf{Gradient descent.} For a step-size $\eta>0$ and symmetric random initialization $\Phi(0)$, the gradient descent works as follows: for $s\in\bZ_+$,
\begin{align*}
    \Phi(s+1)&=\Phi(s)-\eta\cdot\del_\Phi\er\LL(\Phi(s)\RR).
\end{align*}
\item \textbf{Projected gradient descent.} For a step-size $\eta > 0$ and projection radii $\bm{\rho}=(\rho_w,\rho_u,\rho_c)^\top\in\bR^{3}_+$, the projected gradient descent works as follows:
\begin{align*}
    {\Wii}(s+1)&=\pi_{\cB_{1}(\Wii(0),\rho_w/\sqrt{m})}\LL\{\Wii(s)-\eta\partial_\Wii\er(\Phi(s))\RR\},~i=1,\ldots,m,\\
    U_i(s+1)&=\pi_{\cB_{2}(U_i(0),\rho_u/\sqrt{m})}\LL\{U_i(s)-\eta\partial_{U_i}\er(\Phi(s))\RR\},~i=1,\ldots,m,\\
    {{c_i}(s+1)}&={\pi_{\cB_{1}(c_i(0),\rho_c/\sqrt{m})}\LL\{c_i(s)-\eta\partial_{c_i}\er(\Phi(s))\RR\},~i=1,\ldots,m.}
\end{align*}
\item {\textbf{Projected stochastic gradient descent.} Let $\{I_s\sim\mathsf{Unif}(\{1,\ldots,n\}):s\in\bN\}$ be an independent and identically distributed (iid) sequence. For any $s\in\bN$, let $\bm{G}_{\Phi_i}(\Phi(s),s)=[G_{\Wii}(\Phi(s),s),G_{U_i}^\top(\Phi(s),s),G_{c_i}(\Phi(s),s)]^\top,~i=1,\ldots,m$ be defined as
\begin{equation*}
\bm{G}_{\Phi_i}(\Phi(s),s)
    := 2\sum_{t=1}^T\LL(F_t(\X^\uind{I_s};\Phi(s))-F_t^\star(X^\uind{I_s})\RR)\del_{\Phi_i}F_t(\X^\uind{I_s};\Phi(s)).
\end{equation*}
For step-size $\eta > 0$ and projection radii $\bm{\rho}\in\bR_+^3$,
\begin{align*}
    {\Wii}(s+1)&=\pi_{\cB_{1}(\Wii(0),\rho_w/\sqrt{m})}\LL\{\Wii(s)-\eta \cdot G_{\Wii}(\Phi(s),s)\RR\},~i=1,\ldots,m,\\
    U_i(s+1)&=\pi_{\cB_{2}(U_i(0),\rho_u/\sqrt{m})}\LL\{U_i(s)-\eta \cdot G_{U_i}(\Phi(s),s)\RR\},~i=1,\ldots,m,\\
    {c_i}(s+1)&=\pi_{\cB_{1}(c_i(0),\rho_c/\sqrt{m})}\LL\{c_i(s)-\eta \cdot G_{c_i}(\Phi(s),s)\RR\},~i=1,\ldots,m.
\end{align*}
}
\end{itemize}
As a consequence of the projection, it holds surely for all $s\geq 0$ that $$\max_{1\leq i\leq m}|\Wii(s)-\Wii(0)|\leq \frac{\rho_w}{\sqrt{m}},\max_{1\leq i\leq m}\|U_i(s)-U_i(0)\|_2\leq\frac{\rho_u}{\sqrt{m}}\mbox{ and }{\max_{1\leq i\leq m}|c_i(s)-c_i(0)|\leq\frac{\rho_c}{\sqrt{m}}},$$ where all inequalities will be fundamental in establishing the convergence results, and the first inequality will also be important to ensure stability. Also note that $\|\W(s)-\W(0)\|_F\leq\rho_w$, $\|\U(s)-\U(0)\|_F\leq \rho_u$, and $\|c(s)-c(0)\|_2\leq \frac{\rho_c}{\sqrt{m}}$ for all $s\geq 0$. This projection scheme, which is commonly known as the max-norm projection, was proposed in \cite{srebro2004maximum}, and has been a very popular regularization method in practice \cite{srivastava2014dropout, goodfellow2013empirical,goodfellow2013maxout}.

\subsection{Main Results: An Overview}
In the following, we present the main results of this paper informally. For the full statements of these results, we refer to Sec. \ref{sec:analysis}. For details on the realizable dynamical systems in the infinite-width limit, denoted as $\iwl$, see Sec. \ref{sec:kernel}.
\begin{theoreminf}[Theorem \ref{thm:proj-gd}{~(Informal}) (Convergence of projected gradient descent for RNNs)]
    For large \hfill\break enough $\bm{\rho} \succ 0$, for any $\delta\in(0,1)$, $\tau\geq 1$ iterations of projected-gradient descent with the step-size $\eta=\frac{1}{T\sqrt{\tau}}$ yields
    \begin{equation*}
        \min_{0\leq s < \tau}\er(\Phi(s)) \leq \mathsf{poly}(\mu_T,{\|\bm{\rho}\|_2, T})\cO\LL(\frac{1}{\sqrt{\tau}}+\sqrt{\frac{\log(2Tn/\delta)}{m}}\RR)
    \end{equation*}
    with probability at least $1-\delta$ over the random initialization for $\bm{F}^\star\in\iwl$ with some $\alpha \in (0,1)$, where $\mu_T = \cO(1)$ if $\alpha_m=\alpha+\frac{\rho_w}{\sqrt{m}}<\frac{1}{\sigma_1}$ and $\mu_T=e^{\Omega(T)}$ if $\alpha_m \geq \frac{1}{\sigma_1}$.
    \label{thminf:proj-gd}
\end{theoreminf}
In particular, to establish $\min_{0\leq s<\tau}\er(\Phi(s))\leq \epsilon$, one needs an RNN with $m = \Omega\LL({\frac{\log\LL(n/\delta\RR)}{\epsilon^2}}\RR)$ neurons and $\tau = \Omega\LL(\frac{1}{\epsilon^2}\RR)$ iterations of gradient descent.

The above theorem can also be used to bound the empirical risk for the average iterate. 
\begin{corollaryinf}[Corollary \ref{cor:avg-iterate-1} ({Informal})(Average-iterate convergence of projected gradient descent)] 
    For any \hfill \break$\delta\in(0,1)$, projected gradient descent with the step-size choice $\eta=\frac{1}{T\sqrt{\tau}}$ yields
    \begin{equation*}
        \er\LL(\frac{1}{\tau}\sum_{s<\tau}\Phi(s)\RR) \leq \mathsf{poly}(\mu_T,{\|\bm{\rho}\|_2,T})\cO\LL(\frac{1}{\sqrt{\tau}}+\frac{\sqrt{\log(2Tn/\delta)}}{\sqrt{m}}\RR),
    \end{equation*}
    for sufficiently large $\bm{\rho} \succ 0$ with probability at least $1-\delta$ for $\bm{F}^\star\in\iwl$.
    \label{cor:avg-iterate}
\end{corollaryinf}
{
We have the following finite-time guarantees for the stochastic gradient descent.
\begin{corollaryinf}[Corollary \ref{cor:sgd-avg-iterate-1} ({Informal})(Convergence of projected stochastic gradient descent)] 
    For any \hfill \break$\delta\in(0,1)$, projected stochastic gradient descent with the step-size choice $\eta=\frac{1}{T\sqrt{\tau}}$ yields
    \begin{equation*}
        \bE\Big[\er\Big(\frac{1}{\tau}\sum_{s<\tau}\Phi(s)\Big)\Big|\Phi(0)\Big] \leq \mathsf{poly}(\mu_T,{\|\bm{\rho}\|_2,T})\cO\LL(\frac{1}{\sqrt{\tau}}+\frac{\sqrt{\log(2Tn/\delta)}}{\sqrt{m}}\RR),
    \end{equation*}
    for sufficiently large $\bm{\rho} \succ 0$ with probability at least $1-\delta$ for $\bm{F}^\star\in\iwl$.
    \label{cor:sgd-avg-iterate}
\end{corollaryinf}
}
Now, we present the convergence results for (projection-free) gradient descent.
\begin{theoreminf}[Theorem \ref{thm:gd} {(Informal}) (Convergence of gradient descent for RNNs)]
    For any $\epsilon > 0$, there exist $\tau_0 = \Omega\LL(\frac{1}{\epsilon^3}\RR)$ and $m_0=\Omega(\frac{1}{\alpha^2\epsilon }+\frac{\mathsf{poly}(\mu_T,T)\log(2nT/\delta)}{\epsilon^3})$ such that for any $\tau > \tau_0$ and $m>m_0$, gradient descent with step-size $\eta = \cO(\epsilon^2)$ yields
    \begin{equation}
        \min_{0\leq s < \tau}\er(\Phi(s)) \leq \epsilon,
    \end{equation}
    with probability at least $1-\delta$ over the random initialization for $\bm{F}^\star\in\iwl$ with some $\alpha\in(0,1)$, and $\mu_T = \cO(1)$ if $\alpha < \frac{1}{\sigma_1}$ and $\mu_T = e^{\Omega(T)}$ otherwise.
    \label{thminf:gd}
\end{theoreminf}

\begin{remark}
    Below we provide interpretations of the outcomes of the above results.
    \begin{itemize}
        \item \textbf{Bounds on the network size.} In Theorems \ref{thminf:proj-gd}-\ref{thminf:gd}, the number of neurons scales logarithmically in the sample size $n$ and the (inverse) failure probability $1/\delta$. We follow a considerably different approach to prove the above convergence results than the existing analysis in \cite{allen2019convergence}, which requires $m$ to depend on $n$ polynomially with degree as large as 100 to guarantee strong regularity conditions such as Polyak-Łojasiewicz condition to establish convergence. {Avoiding the need to establish Polyak-Łojasiewicz condition results in improved network width bounds, which comes at the expense of sublinear convergence rates.}
        
        
        \item \textbf{Impact of long-term dependencies: exploding gradient problem.} In these results, $\alpha > 0$ is a term that quantifies the long-term dependence of the system at time $t$ to the previous data points $X_k,k<t$. The term $\mu_T$, which is a multiplicative factor for both $m$ and $\tau$, has a cutoff point at $1/\sigma_1$, where $\sigma_1$ is the modulus of Lipschitz continuity of the activation function. Remarkably, if $\alpha<\sigma_1$, then $\mu_T=\cO(1)$ for sufficiently large $m$, while $\mu_T=\exp(\Omega(T))$ for any $m\in\bN$ if $\alpha\geq\sigma_1$. This result captures theoretically the impact of long-term dependencies on the performance of gradient descent, which has been widely observed in practice as the \emph{exploding gradient problem} \cite{goodfellow2016deep, bengio1993problem, bengio1994learning, hochreiter1991untersuchungen}. 

        \item \textbf{Impact of regularization on the stability and convergence rate.} Gradient descent without projection requires $\tau=\Omega(1/\epsilon^3)$ iterations with a neural network of $m=\Omega(1/\epsilon^3)$ neurons to achieve a target error of $\epsilon>0$, while projected gradient descent requires an improved convergence rate with $\tau=\Omega(1/\epsilon^2)$ iterations with a neural network of $m=\Omega(1/\epsilon^2)$ neurons in the same setting. 
        
        
    \end{itemize} 
\end{remark}

\section{Infinite-Width Limit of Recurrent Neural Networks}\label{sec:kernel}
A fundamental question in the study of RNNs is which class of dynamical systems can be \textit{efficiently} (i) approximated, and (ii) learned by RNNs. By exploiting the random feature approach in \cite{rahimi2008uniform,ji2019neural}, we will first characterize the neural tangent kernel (NTK) for RNNs using first principles, and then use the concept of transportation mappings to characterize the class of systems that can be represented by infinite-width limit of randomly-initialized RNNs. The following result facilitates a clutter-free analysis of the neural tangent kernel and the associated reproducing kernel Hilbert space.

\begin{proposition}
    Let $\sigma\in\cC^2(\bR,\bR)$. For any $\mathbf{X}\in\bR^{d\times T}$, $i,j\in\{1,2,\ldots,m\}$ and $t>0$,
     \begin{align*}
    \partial_{U_i}H_t^{(j)}(\X;\Theta) &= \delta_{ij}\sum_{k=0}^{t-1}W_{ii}^kX_{t-k-1}I^\uind{i}_{t-1}(\X;\Theta)\ldots I_{t-k-1}^\uind{i}(\X;\Theta),\\
    \partial_{\Wii}H_t^{(j)}(\mathbf{X};\Theta) &= \delta_{ij}\sum_{k=0}^{t-1}W_{ii}^kH_{t-k-1}^{(i)}(\mathbf{X};\Theta)I^\uind{i}_{t-1}(\X;\Theta)\ldots I^\uind{i}_{t-k-1}(\X;\Theta),
    \end{align*}
    where $\delta_{ij}$ is the Kronecker delta, and $$I_{t}^\uind{i}(\mathbf{X}; \Theta)=\sigma'\big(\Wii H_t^{(i)}(\mathbf{X}; \Theta)+U_i^\top X_t\big),~t\in\bZ_+,i=1,2,\ldots,m.$$ Consequently, {we can characterize the gradients of $F_t$ as 
\begin{align}
\begin{aligned}
    \del_{\Theta_i}F_t(\X;\Phi) &= \frac{c_i}{\sqrt{m}}\sum_{k=0}^{t-1}W^k_{ii} \begin{bmatrix}
    H_{t-k-1}^\uind{i}(\X;\Theta)\\
    X_{t-k-1}
\end{bmatrix}\prod_{s=0}^kI_{t-s-1}^\uind{i}(\X;\Theta),\\
    \del_{c_i}F_t(\X;\Phi) &= \frac{1}{\sqrt{m}}H_t^\uind{i}(\X;\Theta).
    \label{eqn:output-gradient}
    \end{aligned}
\end{align}    
}
    \label{prop:gradient}
\end{proposition}

\begin{remark}
    Note that the result in Proposition \ref{prop:gradient} holds for differentiable activation functions. On the other hand, the result can be extended to ReLU by defining
     $$I_{t}^\uind{i}(\mathbf{X}; \Theta):=\mathbbm{1}\big\{\Wii H_t^{(i)}(\X; \Theta)+U_i^\top X_t \geq 0\big\},$$ since $\sigma'(z)=\ind\{z\geq 0\}$ everywhere except $z=0$. However, RNNs with ReLU activation functions will require a special treatment in the analysis, as ReLU is not a smooth function. Throughout this paper, we concentrate on smooth activation functions of type \eqref{eqn:sigma}.
\end{remark}

\begin{proof}[Proof of Proposition \ref{prop:gradient}]
    For notational simplicity, we will use $I_{t}^\uind{i}$ and $H_t^{(i)}$ throughout the proof. For $t\geq 1$, by the chain rule, \begin{equation}
    \partial_{U_i}H_t^{(j)}=(\Wii \partial_{U_i}H_{t-1}^{(j)}+\delta_{ij}X_{t-1})I_{t-1}^\uind{j}.
    \label{eqn:gradient}
    \end{equation} 
    If $t=1$, then $H_1^{(j)}=\sigma(U_j^\top X_0)$, which would imply $\del_{U_i}H_t^{(j)}=\delta_{ij}X_0I_{0}^\uind{j},$ since $H_0=0$. {For $t > 1$, assuming $\del_{U_i}H_{t-1}^{(j)}=\delta_{ij}\sum_{k=0}^{t-2}W_{ii}^k X_{t-k-2}I_{t-2}^\uind{j}\ldots I_{t-k-2}^\uind{j}$, \eqref{eqn:gradient} implies $$\del_{U_i}H_t^{(j)}=\delta_{ij}I^\uind{j}_{t-1}\big(\sum_{k=0}^{t-2}W_{ii}^{k+1} X_{t-k-2}I^\uind{j}_{t-2}\ldots I^\uind{j}_{t-k-2}+X_{t-1}\big)=\delta_{ij}\sum_{k=0}^{t-1}W_{ii}^kX_{t-k-1}I_{t-1}^\uind{j}\ldots I_{t-k-1}^\uind{j},$$ which would conclude the proof by induction. For $t\geq 1$, by the chain rule, we have $$\del_\Wii H_t^\uind{j}=\LL(\delta_{ij}H_{t-1}^\uind{j}+\Wii\del_\Wii H_{t-1}^\uind{j}\RR)I_{t-1}^\uind{j}.$$ For $t=1$, we have $\del_\Wii H_t^\uind{j}=0$, verifying the base case. Assuming the statement for $t > 1$, $\del_\Wii H_t^\uind{j}=\delta_{ij}I_{t-1}^\uind{j}\big(H_{t-1}^\uind{j}+\sum\limits_{k=0}^{t-2}W_{ii}^{k+1}H_{t-k-2}^\uind{j}I_{t-2}^\uind{j}\ldots I_{t-k-2}^\uind{j}\big)=\delta_{ij}\sum\limits_{k=0}^{t-1}H_{t-k-1}^\uind{j}I_{t-1}^\uind{j}\ldots I_{t-k-1}^\uind{j},$ concluding the proof by induction.}
\end{proof}

\subsection{Neural Tangent Kernel for Elman-Type Recurrent Neural Networks}
The key component of the NTK analysis is the following linear mapping of the learnable parameter:
\begin{equation*}
    F_t^{\mathsf{Lin}}(\X;\Phi)=\sum_{i=1}^m\big\langle\del_{\Theta_i} F_t(\X;\Phi(0)), \Theta_i-\Theta_i(0)\big\rangle{+\sum_{i=1}^m\del_{c_i}F_t(\X;\Phi(0))\LL(c_i-c_i(0)\RR)},
\end{equation*}
for any $\Phi$, where $F_t(\X;\Phi(0))$ is the RNN with a symmetric random initialization $\Phi(0)$. In the following, we characterize the NTK for the Elman-type RNNs.
\begin{proposition}[NTK for Diagonal RNNs]
    Let $\mathbf{X},\mathbf{X}'\in\bR^{d\times T}$, and $\Phi(0)$ be a symmetric random initialization. Then, for any $t\in\{1,2,\ldots,T\}$, we have
{
    \begin{align*}
\lim_{m\rightarrow\infty}\underbrace{\sum_{i=1}^m\begin{bmatrix}\del_{\Theta_i} F_t(\X;\Phi(0))\\\del_{c_i} F_t(\X;\Phi(0))\end{bmatrix}^\top\begin{bmatrix}\del_{\Theta_i} F_t(\X';\Phi(0))\\\del_{c_i} F_t(\X';\Phi(0))\end{bmatrix}}_{=\del_\Phi^\top F_t(\X;\Phi(0))\del_\Phi F_t(\X;\Phi(0))}=\fK_t(\mathbf{X}, \mathbf{X}'),
    \end{align*}}
    almost surely, where
    \begin{equation*}
        \fK_t(\mathbf{X},\mathbf{X}')=\fK_t^w(\X,\X')+\fK_t^u(\X,\X'){+\fK_t^c(\X,\X')},
    \end{equation*}
    with
    \begin{align*}
        \fK_t^u(\X,\X') &= \bE\big[\sum_{k,\ell=0}^{t-1}w_0^{k+\ell}X_{t-k-1}^\top {X}'_{t-\ell-1}\prod_{s=0}^kI_{t-s-1}^0(\mathbf{X})\prod_{s=0}^\ell I_{t-s-1}^0(\mathbf{X}')\big],\\
        \fK_t^w(\X,\X') &= \bE\big[\sum_{k,\ell=0}^{t-1}w_0^{k+\ell}h_{t-k-1}^0(\mathbf{X})h^0_{t-\ell-1}(\mathbf{X}')\prod_{s=0}^kI_{t-s-1}^0(\mathbf{X})\prod_{s=0}^\ell I_{t-s-1}^0(\mathbf{X}')\big],\\
        {\fK_t^c(\X,\X')} &{= \bE[h_t^0(\X)h_t^0(\X')]},
    \end{align*}
    where the expectations are over $w_0\sim\mathsf{Rad}_\alpha, {c_0\sim\mathsf{Rad}_1, u_0\sim\mathcal{N}(0,I_d)}$, with $$I_t^0(\mathbf{X})=\sigma'(w_0h_t^0(\mathbf{X})+u_0^\top X_t),$$ and $h_t^0(\mathbf{X})=\sigma(w_0h_{t-1}^0(\mathbf{X})+u_0^\top X_{t-1})$ for the initial condition $h_0^0(\mathbf{X})=0$.
    \label{prop:ntk}
\end{proposition}
{\begin{proof}
    Using Prop. \ref{prop:gradient} and the properties of symmetric initialization (i.e., $|c_i(0)|=1$ for all $i \leq m$, and $\Theta_i(0)=\Theta_{i+\frac{m}{2}}(0)$ for $i\leq \frac{m}{2}$), we have 
    \begin{multline}\sum_{i=1}^m\del_{U_i}^\top F_t(\X;\Phi(0))\del_{U_i} F_t(\X';\Phi(0))\\=\frac{2}{m}\sum_{i=1}^{m/2}\sum_{k,\ell=0}^{t-1}W_{ii}^{k+\ell}(0)X_{t-k-1}^\top X'_{t-\ell-1}\prod_{s=0}^kI_{t-s-1}^\uind{i}(\X;\Theta(0))\prod_{s=0}^\ell I_{t-s-1}^\uind{i}(\X';\Theta(0)).
    \label{eqn:ntrf}
    \end{multline}
    Since the summand for each $i=1,\ldots,\frac{m}{2}$ in \eqref{eqn:ntrf} is bounded and iid, strong law of large numbers implies $\del_{U_i}^\top F_t(\X;\Phi(0))\del_{U_i} F_t(\X';\Phi(0))\rightarrow \kappa_t^u(\X,\X')$ almost surely as $m\rightarrow\infty$. The proofs for the almost-sure convergence of $\del_{\Wii}F_t(\X;\Phi(0))\del_{\Wii}F_t(\X';\Phi(0))$ to $\kappa_t^w(\X,\X')$ and $\del^\top_{c_i}F_t(\X;\Phi(0))\del_{c_i}F_t(\X';\Phi(0))$ to $\kappa_t^c(\X,\X')$ follow from identical arguments.
\end{proof}}
It is easy to verify that the neural tangent kernel (NTK) $\fK_t$ is a Mercer kernel for every $t\in\bN$, therefore it has an associated reproducing kernel Hilbert space (RKHS) $\cH_{\fK_t}$. $\del_{\Phi} F_t(\cdot;\Phi(0))$ acts as a random feature at a finite-width $m$, and leads to the NTK in the infinite-width limit. 

\begin{remark}[Reduction to feedforward neural networks]
    Note that in the specific case of $\alpha = 0,\nb_w=0$ and $T=1$, we obtain the neural tangent kernel for the single-hidden-layer feedforward neural networks with smooth activation functions. In that respect, our work generalizes the kernel analysis for shallow neural networks to RNNs with memory.
\end{remark}

The class of dynamical systems that can be represented in these RKHSs can be characterized by using the concept of transportation mappings \cite{rahimi2008uniform, ji2019neural}.

\subsection{Infinite-Width Limit of RNNs}
Let $\bm{v}_c,\bm{v}_w:\bR\rightarrow\bR$ and $\bm{v}_u:\bR^d\rightarrow\bR^d$ be transportation mappings such that
\begin{align*}
    \bE[|\bv_w(w_0)|^2]&=\frac{1}{2}\Big(|\bv_w(\alpha)|^2+|\bv_w(-\alpha)|^2\Big)<\infty,\\
    \bE[\|\bv_u(u_0)\|_2^2]&=\frac{1}{(2\pi)^{d/2}}\int_{\bR^d}\|\bv_u(u)\|_2^2\exp(-\frac{1}{2}\|u\|_2^2)du<\infty,\\
    {\bE[|\bv_c(c_0)|^2]}&{=\frac{1}{2}\Big(|\bv_c(1)|^2+|\bv_c(-1)|^2\Big)<\infty}.
\end{align*}
{Then, given a transportation mapping $\bv(\phi) := \begin{bmatrix}
    \bv_w(w)\\\bv_u(u)\\\bv_c(c)
\end{bmatrix}$ for $\phi=\begin{bmatrix}
    w\\u\\c
\end{bmatrix}\in\bR^{d+2}$ with $w,c\in\bR$ and $u\in\bR^d$, we consider mappings of the form
\begin{equation}
    \tilde{F}_t(\X,\bv)=\bE\Big[\sum_{k=0}^{t-1}w_0^k\begin{bmatrix}
        \bv_w(w_0)\\\bv_u(u_0)
    \end{bmatrix}^\top \begin{bmatrix}
        h_{t-k-1}^0\\X_{t-k-1}
    \end{bmatrix}\prod_{s=0}^kI_{t-s-1}^0+h_t^0\bv_c(c_0)\Big].
    \label{eqn:inf-width}
\end{equation}}
It can be easily shown that, given any $t\in\bZ_+$, the completion of the function class $$\{[\mathbf{X}]_t\mapsto\tilde{F}_t([\mathbf{X}]_t;\bv):\bE|\bv_w(w_0)|^2<\infty,\bE\|\bv_u(u_0)\|_2^2<\infty,\bE|\bv_c(c_0)|^2<\infty\},$$ is the RKHS associated with the NTK $\cH_{\fK_t}$ \cite{rahimi2008uniform, ji2019neural}. For any $t\in \{1,2,\ldots,T\}$, the inner product of two functions in $\cH_{\fK_t}$ associated with the transportation mappings $\bv$ and $\bv'$ is $$\LL\langle \tilde{F}_t(\cdot,\bv),\tilde{F}_t(\cdot,\bv')\RR\rangle_{\cH_{\fK_t}}=\bE\LL[\langle\bv(\phi_0),\bv'(\phi_0)\rangle\RR]~\mbox{where}~\phi_0^\top=(w_0,u_0^\top,c_0).$$ As such, the RKHS norm of any $f\in\cH_{\fK_t}$ is $\|f\|_{\cH_{\fK_t}}=\sqrt{\bE[\|\bv(\phi_0)\|^2]}$. Furthermore, since the kernel $\fK_t$ is universal, the associated RKHS $\cH_{\fK_t}$ is dense in the space of continuous functions on a compact set \cite{ji2019neural}.

Upon these observations, we consider the class of dynamical systems that map the input sequence $\mathbf{X}$ to $\tilde{\bm{F}}(\mathbf{X};\bv)=(\tilde{F}_t(\X;\bv))_{1\leq t\leq T}$ where $\tilde{F}_t\in\mathscr{H}_{\kappa_t}$ with sup-norm constrained $\bv$:
\begin{equation}
    \cF_{\bm{\nb}}=\Big\{\X\mapsto\tilde{F}_t(\X;\bv):\bv\in\mathcal{V}_{\bm{\nb}},t=1,\ldots,T\Big\},
    \label{eqn:class}
\end{equation}
for $\bm{\nb}=(\nb_w,\nb_u,\nb_c)^\top\in\bR_{+}^3$ where $$\mathcal{V}_{\bm{\nb}}:=\{\bv:\bR^{d+2}\rightarrow\bR^{d+2}:\sup_{w_0\in\bR}|v_w(w_0)|\leq \bar{\nu}_w,\sup_{u_0\in\bR^d}\|v_u(u_0)\|_2\leq \bar{\nu}_u,\sup_{c\in\bR}|\bv_{c}(c)|\leq \nb_c\}.$$ $\tF_t$ is defined in \eqref{eqn:inf-width}. It is straightforward to see that, for any $t\in\{1,2,\ldots,T\}$, we have $\|\tilde{F}_t\|_{\cH_{\fK_t}}\leq \|\bm{\nb}\|$, and this constraint on the sup-norm will provide significant simplification in the analysis of the gradient descent algorithm with finite-width neural networks.

\subsection{Approximating $\cF_{\nb}$ by Randomly-Initialized RNNs}
In this subsection, we provide approximation results concerning the approximation of $\iwl$ for any given $\bm{\nb}$ by using randomly-initialized RNNs of a finite number of neurons. Let
\begin{align}
        \mu_{0,t}(z) := \sum_{k=0}^{t-1}z^k,~\mbox{and}~\mu_{1,t}(z) := \sum_{k=0}^{t-1}(k+1)z^k,
        \label{eqn:zeta}
\end{align}
for any $z\in\bR_\gz$ and $t\in \bN$. Also, an important concept throughout this work is compactly-supported set of parameters {$\Omega_{\bm{\rho}},~\bm{\rho}=(\rho_w,\rho_u,\rho_c)^\top\in\bR^3_+$,} where
\begin{align*}
    \Omega_{\bm{\rho}} = \prod_{i=1}^m\LL(\mathcal{W}_{\rho_w,i}\times\mathcal{U}_{\rho_u,i}\times\mathcal{C}_{\rho_c,i}\RR),
\end{align*}
for the sets $\mathcal{W}_{\rho_w,i} = \LL\{w\in\bR:|w-\Wii(0)|\leq \frac{\rho_w}{\sqrt{m}}\RR\},
    \mathcal{U}_{\rho_w,i} = \LL\{u\in\bR^d:\|u-U_i(0)\|\leq \frac{\rho_u}{\sqrt{m}}\RR\}$ and $\mathcal{C}_{\rho_c,i} = \LL\{c\in\bR:|c-c_i(0)|\leq \frac{\rho_c}{\sqrt{m}}\RR\}$ for $i=1,2,\ldots,m$.


We start with bounding the approximation error for the linear model $\Lin$.
\begin{proposition}[Approximating $\tilde{F}_t$ with $F_t^\mathsf{Lin}$]
    Let $\tilde{\bm{F}}(\cdot;\bv)\in\cF_{\bm{\nb}}$ for $\bm{\nb}\in\bR_+^3$, $\Phi(0)$ be a symmetric random initialization, and $\mathbf{X}\in\bR^{d\times T}$ with $\max_{1\leq t\leq T}\|X_t\|_2\leq 1$. Then for any $\delta \in (0,1)$ and $t\in\{1,2,\ldots,T\}$,{
    \begin{equation*}
        \inf_{\Phi\in\Omega_{\bm{\nb}}}\Big|\tilde{F}_t(\X;\bv)-F_t^\mathsf{Lin}(\X;\Phi)\Big| \leq 2\|\bm{\nb}\|_2\Big(\sigma_1\cdot\mu_{0,t}(\alpha\sigma_1)\cdot(1+\sigma_0)+\sigma_0\Big)\sqrt{\frac{\log(\frac{2}{\delta})}{m}},
    \end{equation*}}
    holds with probability at least $1-\delta$ over the random initialization.
    \label{prop:lin-ntrf}
\end{proposition}
\noindent We note that the above approximation result holds pointwise for each $\X$, thus one needs to use uniform convergence bounds to ensure that it holds for all data points $\X^\uind{j},j\in\{1,2,\ldots,n\}$.

\begin{proof}[Proof of Proposition \ref{prop:lin-ntrf}]
    Given $\tilde{\bm{F}}(\cdot; \bv)$ with $\bv\in\mathcal{V}_{\bm{\nb}}$ and random initialization $\Phi(0)$, let 
    \begin{align}
        \bar{\Phi}_i = \begin{bmatrix}
            \bar{W}_{ii}\\
            \bar{U}_i\\
            \bar{c}_i
        \end{bmatrix} = \Phi_i(0) + \frac{1}{\sqrt{m}}\begin{bmatrix}
            c_i(0)\bv_w(\Wii(0))\\
            c_i(0)\bv_u(U_i(0))\\
            \bv_c(c_i(0))
        \end{bmatrix},
        \label{eqn:projection}
    \end{align}
    which implies $|\bar{W}_{ii}-\Wii(0)|\leq \frac{|\bv_w(\Wii(0))|}{\sqrt{m}}\leq \frac{\nb_w}{\sqrt{m}}$, $\|\bar{U}_i-U_i(0)\|_2\leq \frac{\|\bv_u(U_i(0))\|_2}{\sqrt{m}}\leq \frac{\nb_u}{\sqrt{m}}$ and $|\bar{c}_i-c_i(0)|\leq \frac{|\bv_c(c_i(0))|}{\sqrt{m}}\leq \frac{\nb_c}{\sqrt{m}}$ for any $i=1,2,\ldots,m$, thus $\bar{\Phi}\in\Omega_{\bm{\nb}}$. Consequently, we have
    \begin{align*}
        \Lin(\X;\bar{\Phi})=\frac{1}{m}\sum_{i=1}^mZ_t^\uind{i}(\X;\Phi(0)),
    \end{align*}
    where {
        $Z_t^\uind{i}(\X;\Phi(0))=\bv^\top(\Phi_i(0))\begin{bmatrix}
            \sum\limits_{k=0}^{t-1}W^k_{ii}(0)\begin{bmatrix}
            H_{t-k-1}^\uind{i}(\X;\Theta(0))\\
            X_{t-k-1}
        \end{bmatrix}\prod\limits_{s=0}^kI^\uind{i}_{t-s-1}(\X;\Theta(0)) \\
        H_t^\uind{i}(\X;\Theta_i(0))
        \end{bmatrix}.$}
    Note that $\bE\Zt=\tilde{F}_t(\X;\bv)$, and 
    {
    \begin{align}\nonumber |Z_t^\uind{i}(\X;\Phi(0))|&\leq \|\bm{\nb}\|_2\LL(\sigma_1\sum_{k=0}^{t-1}(\alpha\sigma_1)^k(\|X_{t-k-1}\|_2+\sigma_0)+\sigma_0\RR),\\
    &\leq \|\bm{\nb}\|_2\big(\sigma_1\cdot\mu_{0,t}(\alpha\sigma_1)\cdot(1+\sigma_0)+\sigma_0\big)~a.s.
    \label{eqn:y-bound}
    \end{align}}
    for all $i\in\{1,\ldots,m\}$, which follows from Cauchy-Schwarz and triangle inequalities. Furthermore, $\Zt$ is an independently and identically distributed (iid) sequence for $i\leq m/2$, where the second half is a replica of the first half due to the symmetric initialization. As such,
    \begin{align*}
        \Lin(\X;\bar{\Phi})-\tilde{F}_t(\X;\bv)&=\frac{1}{m}\sum_{i=1}^m\LL(\Zt-\bE\Zt\RR)\\
        &=\frac{2}{m}\sum_{i=1}^{m/2}\LL(\Zt-\bE\Zt\RR).
    \end{align*}
    Hence, by using Hoeffding's inequality, for any $\delta\in(0,1)$,
    \begin{equation*}
        |\Lin(\X;\bar{\Phi})-\tilde{F}_t(\X;\bv)| \leq 2\|\bm{\nb}\|_2\big(\sigma_1\cdot\mu_{0,t}(\alpha\sigma_1)\cdot(1+\sigma_0)+\sigma_0\big)\sqrt{\frac{\log(\frac{2}{\delta})}{m}},
    \end{equation*}
    with probability at least $1-\delta$.
\end{proof}

\begin{proposition}[Bounding the linearization error]
    Let $\Phi\in\Omega_{\bm{\rho}}$ for some $\rho\in\bR_+^3$. Then,
    \begin{equation*}
        |F_t(\X;\Phi)-\Lin(\X;\Phi)| \leq \frac{2}{\sqrt{m}}\Big(\Lambda_t^2\sigma_2+\gamma_t\sigma_1\Big)\|\Theta-\Theta(0)\|_2^2+{\frac{2L_t}{\sqrt{m}}\|c-c(0)\|_2\|\Theta-\Theta(0)\|_2},
    \end{equation*}
    surely for any $t\in\{1,2,\ldots,T\}$, where $\Lambda_t$ and $\gamma_t$ are defined in Lemma \ref{lemma:smoothness} such that {the functions $\Theta_i\mapsto\Wii H_t^\uind{i}(\X;\Theta)$ are $\Lambda_t$-Lipschitz and $\gamma_t$-smooth.}
    \label{prop:lin-ntk}
\end{proposition}

\begin{proof}[Proof of Prop. \ref{prop:lin-ntk}]
    For any $\Phi \in \bR^{m(d+2)}$, let \begin{align}
        \label{eqn:pact}\vpt(\Theta)&:=\Wii H_{t}^\uind{i}(\X;\Theta),\\
        \nonumber \vvpt(\Theta) &:= \Wii H_t^\uind{i}(\X;\Theta)+U_i^\top {X_{t}},\\
        \nonumber \check{\varepsilon}_t^\mathsf{Lin}(\Phi)&:=F_t(\X; \Phi)-\Lin(\X; \Phi).
    \end{align}
    {Note that $$F_t^\mathsf{Lin}(\X;\Phi)=\sum_{i=1}^m\frac{1}{\sqrt{m}}c_i(0)\LL\langle\del_{\Theta_i} H_t^\uind{i}(\X;\Theta(0)),\Theta_i-\Theta_i(0)\RR\rangle+\sum_{i=1}^m\frac{1}{\sqrt{m}}(c_i-c_i(0))H_t^\uind{i}(\X;\Theta(0)).$$
    Then, we have the following decomposition from 
    \begin{align*}
        \check{\varepsilon}_t^\mathsf{Lin}(\Phi) &= \underbrace{\frac{1}{\sqrt{m}}\sum_{i=1}^mc_i(0)\LL(H_t^\uind{i}(\X;\Theta)-\LL\langle\del_{\Theta_i} H_t^\uind{i}(\X;\Theta(0)),\Theta_i-\Theta_i(0)\RR\rangle\RR)}_{:= \check{\varepsilon}_{t,1}^\mathsf{Lin}(\Phi)} \\&+ \underbrace{\frac{1}{\sqrt{m}}\sum_{i=1}^m(c_i-c_i(0))\LL(H_t^\uind{i}(\X;\Theta)-H_t^\uind{i}(\X;\Theta(0))\RR)}_{:= \check{\varepsilon}_{t,2}^\mathsf{Lin}(\Phi)}.
    \end{align*}
    }
    {First, note that $\Theta_i\mapsto H_t^\uind{i}(\X;\Theta)$ is $L_t$-Lipschitz from Lemma \ref{lemma:lipschitz}. Therefore, we have
    \begin{align*}
        |\check{\varepsilon}_{t,2}^\mathsf{Lin}(\Phi)| &\leq \frac{1}{\sqrt{m}}\sum_{i=1}^m|c_i-c_i(0)|\LL|H_t^\uind{i}(\X;\Theta)-H_t^\uind{i}(\X;\Theta(0))\RR|\\ 
        &\leq \frac{L_t}{\sqrt{m}}\sum_{i=1}^m|c_i-c_i(0)|\cdot \|\Theta_i-\Theta_i(0)\|_2 \leq \frac{L_t}{\sqrt{m}}\|c-c(0)\|_2\cdot\|\Theta-\Theta(0)\|_2,
    \end{align*}
    where we used triangle and Cauchy-Schwarz inequalities in the first and last inequalities above, respectively. In the following, we bound $\check{\varepsilon}_{t,1}^\mathsf{Lin}(\Phi)$. Let $\Theta^0=\Theta(0)$ as a shorthand notation.}
    \begin{align}
        \nonumber \check{\varepsilon}_{t,1}^\mathsf{Lin}(\Phi)&=\sum_{i=1}^m\frac{c_i(0)}{\sqrt{m}}\LL((\sigma(\qact{t-1}(\Theta))-\sigma(\qact{t-1}(\Theta^0))-\sigma'(\qact{t-1}(\Theta^0))\del_{\Theta_i}^\top \qact{t-1}(\Theta^0)(\Theta_i-\Theta_i(0))\RR)
        \\ \nonumber &=\sum_{i=1}^m\frac{c_i(0)}{\sqrt{m}}\underbrace{\LL(\sigma(\qact{t-1}(\Theta))-\sigma(\qact{t-1}(\Theta^0))-\sigma'(\qact{t-1}(\Theta^0))\LL[\qact{t-1}(\Theta)-\qact{t-1}(\Theta^0)\RR]\RR)}_{A_{1,i}}
        \\&+\frac{1}{\sqrt{m}}\sum_{i=1}^mc_i(0)\sigma'(\qact{t-1}(\Theta^0))\underbrace{\LL(\pact{t-1}(\Theta)-\pact{t-1}(\Theta^0)-\del_{\Theta_i}^\top\pact{t-1}(\Theta^0)(\Theta_i-\Theta_i(0))\RR)}_{A_{2,i}}.
        \label{eqn:lino}
    \end{align}

For the first term, we have 
    \begin{align}
        \nonumber |A_{1,i}| &\overset{(\spadesuit)}{\leq} \frac{\sigma_2}{2}\LL|\qact{t-1}(\Theta)-\qact{t-1}(\Theta(0))\RR|^2 
        \\ \nonumber  &\overset{(\clubsuit)}{\leq} \sigma_2\LL(\LL|\pact{t-1}(\Theta)-\pact{t-1}(\Theta(0))\RR|^2+\|U_i-U_i(0)\|_2^2\RR)
        \\ &\overset{(\diamondsuit)}{\leq}\sigma_2\LL(\Lambda_t^2\|\Theta_i-\Theta_i(0)\|_2^2+\|U_i-U_i(0)\|_2^2\RR),
        \label{eqn:a1-lin}
    \end{align}
    where $(\spadesuit)$ follows from $\sigma_2$-smoothness of the activation function $\sigma$ and {Taylor's theorem}, $(\clubsuit)$ follows from $\max_{t=1,\ldots,T}\|X_t\|_2\leq 1$ and the inequality $(x+y)^2\leq 2(x^2+y^2)$, and finally $(\diamondsuit)$ follows from $\Lambda_t$-Lipschitz continuity of $\pact{t-1}$, which is established in Lemma \ref{lemma:smoothness}.

    For the second term,
    \begin{align}
        |A_{2,i}|\leq \frac{\gamma_t}{2}\|\Theta_i-\Theta_i(0)\|_2^2,
        \label{eqn:a2-lin}
    \end{align}
    from $\gamma_t$-smoothness of $\Theta_i\mapsto \pact{t-1}(\Theta)$, which is proved in Lemma \ref{lemma:smoothness}. By substituting the bounds in \eqref{eqn:a1-lin} and \eqref{eqn:a2-lin} into \eqref{eqn:lino}, and by using $|c_i(0)|=1$ and $\max_{z\in\bR}|\sigma'(z)|\leq \sigma_1$,
    \begin{align*}
        |\check{\varepsilon}_{t,1}^\mathsf{Lin}(\Phi)| &\leq \frac{1}{\sqrt{m}}\sum_{i=1}^m\LL(\sigma_2(1+\Lambda_t^2)\|U_i-U_i(0)\|_2^2+\sigma_2\Lambda_t^2|\Wii-\Wii(0)|^2+\sigma_1\gamma_t\|\Theta_i-\Theta_i(0)\|_2^2\RR)\\
        &\leq \frac{1}{\sqrt{m}}\LL(\sigma_2(1+\Lambda_t^2)\|\Theta_i-\Theta_i(0)\|_2^2+\sigma_1\gamma_t\|\Theta_i-\Theta_i(0)\|_2^2\RR).
    \end{align*}
    The result follows by $|\check{\varepsilon}_t^\mathsf{Lin}(\Phi)| \leq |\check{\varepsilon}_{t,1}^\mathsf{Lin}(\Phi)|+|\check{\varepsilon}_{t,2}^\mathsf{Lin}(\Phi)|$ and the bounds for $|\check{\varepsilon}_{t,1}^\mathsf{Lin}(\Phi)|$ and $|\check{\varepsilon}_{t,2}^\mathsf{Lin}(\Phi)|$.
\end{proof}

\begin{remark}[Challenges in bounding the linearization error for RNNs]
    The weight-sharing in RNNs, which implies that the same parameter $\Phi$ is used at every time-step $t\in\{1,2,\ldots,T\}$, enables RNNs to learn dynamical systems, however this comes at the expense of a complicated analysis. In the case of feedforward neural networks, $\Theta_i\mapsto\qact{t-1}(\Theta)$ is a linear mapping in $\Theta_i$, thus smoothness of the activation function $\sigma$ would directly yield the linearization error \cite{mjt_dlt}. In contrast, in our case, $\Theta_i\mapsto \qact{t-1}(\Theta)$ is highly nonlinear due to the weight-sharing in RNNs, thus it is uncertain how to establish the upper bound on the linearization error $|\epsilon_t^\mathsf{Lin}(\Theta)|$. Our analysis reveals that the smoothness of $\Theta_i\mapsto \pact{t-1}(\Theta)=\Wii H_{t-1}^\uind{i}(\X;\Theta)$ rather than its linearity suffices to establish this result, thus resolving this significant challenge in RNNs. 
    
\end{remark}

Using Prop. \ref{prop:lin-ntrf}-\ref{prop:lin-ntk}, we can obtain an approximation error bound for approximating a dynamical system $\tilde{\bm{F}}\in\cF_{\bm{\nb}}$ by using a randomly-initialized RNN of width $m$ of the form \eqref{eqn:rnn}.

\begin{proposition}
    Let $\bm{\nb}\in\bR_+^3$. For any $\tilde{\bm{F}}\in\cF_{\bm{\nb}}$, $\delta\in(0,1)$, training set $\cS$ and width $m\in\bN$, there exists $\Phi\in\Omega_{\bm{\nb}}$ such that
    {\begin{align*}
        \Big|\tilde{F}_t(\X^\uind{j};\bm{\nu})-F_t(\X^\uind{j};\Phi)\Big| &\leq 2\|\bm{\nb}\|_2\Big(\sigma_1\cdot\mu_{0,t}({\alpha}\sigma_1)\cdot(1+\sigma_0)+\sigma_0\Big)\sqrt{\frac{\log(\frac{2nT}{\delta})}{m}}\\&+\frac{2}{\sqrt{m}}\Big(\Lambda_T^2\sigma_2+\gamma_T\sigma_1\Big)(\nb_u^2+\nb_w^2) + \frac{2L_t}{\sqrt{m}}\nb_c\sqrt{\nb_u^2+\nb_w^2},
    \end{align*}}
holds simultaneously for all $j\in\{1,\ldots,n\}$ and $t\in\{1,\ldots,T\}$ with probability at least $1-\delta$ over the initialization for $\Lambda_t,\gamma_t$ given in Lemma \ref{lemma:smoothness}.
    \label{thm:approx}
\end{proposition}

\begin{proof}[Proof of Prop. \ref{thm:approx}]
Let $\bar{\Phi}$ be defined as \eqref{eqn:projection}, which ensures that $\bar{\Phi}\in\Omega_{\bm{\nb}}$. For any $j\in\{1,2,\ldots,n\},~t\in\{1,2,\ldots,T\}$ and $\delta\in(0,1)$, by using triangle inequality,
\begin{align*}
\Big|\tilde{F}_t(\X^\uind{j};\bm{\nb})-F_t(\X^\uind{j};\bar{\Phi})\Big| \leq \Big|\tilde{F}_t(\X^\uind{j};\bm{\nb})-\Lin(\X^\uind{j};\bar{\Phi})\Big|+\Big|\Lin(\X^\uind{j};\bar{\Phi})-F_t(\mathbf{X}^\uind{j};\bar{\Phi})\Big|.
\end{align*}
  By using Propositions \ref{prop:lin-ntrf}-\ref{prop:lin-ntk} to bound the terms on the right-hand side of the above inequality, and the union bound for $j\in\{1,\ldots,n\}$ {and $t=1,\ldots,T$}, the result follows.
\end{proof}


\begin{remark}[Transportation mappings and approximation]
    Given a dynamical system $\tilde{\bm{F}}\in\iwl$, the construction $\bar{\Phi}$ in \eqref{eqn:projection} yields the approximation bounds in Proposition \ref{thm:approx}. As such, the mapping $\bm{\nb}$ transforms a randomly chosen parameter $\Theta(0)$ to a point closer to the projection of $\tilde{\bm{F}}$ onto the class of finite-width, randomly-initialized dynamical systems with high probability, which is the reason they are called transportation mappings.
    \label{rem:transport}
\end{remark}


\section{Convergence of GD for RNNs: Rates and Analysis}\label{sec:analysis}
Continuity and smoothness properties of $\Theta_i\mapsto H_t^\uind{i}(\X_t;\Theta)$ as well as $\Theta_i\mapsto \Wii H_t^\uind{i}(\X;\Theta)$ play an important role in establishing the convergence rates for RNNs under gradient descent methods. 
\subsection{Local Lipschitz Continuity and Smoothness of the Hidden State}
In the following, we first state the results on these properties to set the notation and conceptual background, then we present the main convergence results in full detail.
            
         The following results will be important in the characterization of long-term dependencies.
\begin{lemma}[Local Lipschitz continuity and smoothness of $\Theta_i\mapsto H_t^\uind{i}$]
    For $\rho\in\bR_+^3$, let $\alpha_m:=\alpha+\frac{\rho_w}{\sqrt{m}}$. Then,
        \begin{enumerate}
            \item $\Theta_i\mapsto H_t^\uind{i}(\X;\Theta)$ is (locally) $L_t$-Lipschitz in $\Omega_{\bm{\rho}}$ with $L_t=\ms^2(\alpha_m\sigma_1)\cdot(\sigma_0^2+1)\cdot\sigma_1^2,$
            \item $\Theta_i\mapsto H_t^\uind{i}(\X;\Theta)$ is (locally) $\beta_t$-smooth in $\Omega_{\bm{\rho}}$ with $\beta_t=\cO\big(d\cdot \ms(\alpha_m\sigma_1)\cdot \mb(\alpha_m\sigma_1)\big).$
        \end{enumerate}
        \label{lemma:lipschitz}
\end{lemma}
\noindent The complete expression for $\beta_t$ and the proof can be found in Appendix \ref{app:ht-lipschitz}.

An important structural result that would be critical in proving the convergence of gradient descent and linearization in the kernel regime is the local Lipschitz continuity and smoothness of $\pact{t}$, which is defined in \eqref{eqn:pact}, over a compact subset around the initial parameter $\Theta(0)$.
\begin{lemma}[Local Lipschitz continuity and smoothness of $\Theta_i\mapsto \pact{t}(\Theta)$]
    For $\bm{\rho}\in\bR_+^3$, let $\alpha_m:=\alpha+\frac{\rho_w}{\sqrt{m}}$. Then,
        \begin{enumerate}
            \item $\Theta_i\mapsto \pact{t}(\Theta)$ is (locally) $\Lambda_t$-Lipschitz in $\Omega_{\bm{\rho}}$ with $\Lambda_t=\sqrt{2}(\sigma_0+1+\alpha_m L_t),$
            \item $\Theta_i\mapsto \pact{t}(\Theta)$ is (locally) $\gamma_t$-smooth in $\Omega_{\bm{\rho}}$ with $\gamma_t=\sqrt{2}\LL(L_t+\alpha_m\beta_t\RR),$
        \end{enumerate}
        where $L_t,\beta_t$ are given in Lemma \ref{lemma:lipschitz}.
        \label{lemma:smoothness}
\end{lemma}

We have the following bounds on the neural network output.
\begin{proposition}[Norm bounds on $F_t$ and $F_t^\star$]
    Given any $\bm{\rho}\in\bR_+^3$, let $\Phi \in \Omega_{\bm{\rho}}$. Then,
    \begin{align*}
        \sup_{\substack{\X\in\bR^{d\times T}\\\|X_i\|_2\leq 1,i=1,\ldots,m}}\max_{t=1,2,\ldots,T}\LL|F_t(\X;\Phi)\RR|\leq L_T\sqrt{\rho_w^2+\rho_u^2} {+ \sigma_0\rho_c} =: F_T^{\mathsf{max}}(\bm{\rho}),
    \end{align*}
    where $L_T\geq L_t$ is a modulus of Lipschitz continuity of $\Theta_i\mapsto H_t^\uind{i}(\X;\Theta)$ on $\Omega_{\bm{\rho}}$ for $t=1,\ldots,T$.

    Let $(F_t^\star)_{t\in\bZ_+}\in\iwl$ for $\bm{\nb}\in\bR_+^3$. Then, we also have $$\sup_{{\X\in\bR^{d\times T}:\|X_i\|_2\leq 1,\forall i}}~\max_{t=1,2,\ldots,T}~|F_t^\star(\X)|\leq \|\bm{\nb}\|_2\big(\sigma_1\cdot\mu_{0,t}(\alpha\sigma_1)\cdot(1+\sigma_0)+\sigma_0\big) =: y_T^{\mathsf{max}}(\bm{\nb}).$$
    \label{prop:norm-bounds}
\end{proposition}
As a shorthand notation, we denote $\varphi_T^\mathsf{max}(\bm{\rho},\bm{\nb}):=F_T^\mathsf{max}(\bm{\rho})+y_T^\mathsf{max}(\bm{\nb}).$
\begin{proof}[Proof of Prop. \ref{prop:norm-bounds}]
    {As a result of the symmetric initialization, we have $F_t(\X;\Phi(0))=0$ for $t=1,\ldots,T.$ Also, we have $\Phi(0)\in\Omega_{\bm{\rho}}$. Thus, by using (local) Lipschitz continuity of $\Theta_i\mapsto H_t^\uind{i}(\cdot;\Theta)$ given in Lemma \ref{lemma:lipschitz}, for all $t=1,2,\ldots,T$, we have 
    \begin{align*}
        |F_t(X;\Phi)| &= |F_t(\X;\Phi)-F_t(\X;\Phi(0))|\\
        &\leq \LL|\frac{1}{\sqrt{m}}\sum_{i=1}^mc_i(0)\LL(H_t^\uind{i}(\X;\Theta)-H_t^\uind{i}(\X;\Theta(0))\RR)\RR|+\Big|\frac{1}{\sqrt{m}}\sum_{i=1}^m(c_i-c_i(0))H_t^\uind{i}(\X;\Theta)\Big|\\ 
        &\leq \frac{1}{\sqrt{m}}\sum_{i=1}^m|H_t^\uind{i}(\X;\Theta)-H_t^\uind{i}(\X;\Theta(0))| + \frac{\sigma_0}{\sqrt{m}}\sum_{i=1}^m|c_i-c_i(0)|\\
        &\leq \frac{L_t}{\sqrt{m}}\sum_{i=1}^m \|\Theta_i-\Theta_i(0)\|_2 + \frac{\sigma_0}{\sqrt{m}}\sum_{i=1}^m\frac{\rho_c}{\sqrt{m}} \leq L_t\sqrt{\rho_w^2+\rho_u^2} + \sigma_0\rho_c.
    \end{align*}
    We conclude the first proof by the monotonicity of $L_t$, which implies $L_t\leq L_T$ for all $t\leq T$.
    
    The second part of the claim follows from \eqref{eqn:y-bound}. Note that
    \begin{align*}
        F_t^\star(\X) = \bE[Z_t^\uind{1}(\X;\Phi(0))],
    \end{align*}
    thus $|F_t^\star(\X)| \leq \bE|Z_t^\uind{1}(\X;\Phi(0))|\leq \|\bm{\nb}\|_2\big(\sigma_1\cdot\mu_{0,t}(\alpha\sigma_1)\cdot(1+\sigma_0)+\sigma_0\big)$ by \eqref{eqn:y-bound}.}
\end{proof}

\subsection{Convergence of Projected Gradient Descent for RNNs}

We provide the formal upper bound for the empirical risk under projected gradient descent in the following.
\begin{theorem}[Convergence of projected gradient descent for RNNs]
    Let $(F_t^\star)_{t\in\bZ_+}\in\iwl$ for $\bm{\nb}\in\bR_+^3$. Then, with projection radius $\bm{\rho}\succeq \bm{\nb}$, projected gradient descent with step-size $\eta = \frac{1}{T\sqrt{\tau}}$ yields for any $\delta\in(0,1)$ and $\tau\geq 1$,
    \begin{equation*}
        \min_{0\leq s<\tau}\er(\Phi(s))\leq \frac{T}{\sqrt{\tau}}\LL(\frac{\|\bm{\nb}\|_2^2}{4}+\Big(\bd(\bm{\rho},\bm{\nb})\big(1+\frac{\rho_c}{\sqrt{m}}\big)L_T + \sigma_0\Big)^2\RR)+\frac{T\bd(\bm{\rho},\bm{\nb})}{\sqrt{m}}\mathsf{err},
    \end{equation*}
    with probability at least $1-\delta$ over the random initialization, where
    \begin{multline}
        \mathsf{err}={(\|\bm{\nb}\|_2+\|\bm{\rho}\|_2)\LL((\beta_T+L_T)\sqrt{\rho_w^2+\rho_u^2}+L_T\rho_c\RR)}+{L_T\rho_c\sqrt{\rho_w^2+\rho_u^2}}\\+\Big(\Lambda_T^2\sigma_2+\gamma_T\sigma_1\Big)(\rho_w^2+\rho_u^2)+\|\bm{\nb}\|_2\big(\sigma_1\mu_{0,T}(\alpha_m\sigma_1)(1+\sigma_0)+\sigma_0\big)\sqrt{\log\LL(\frac{2nT}{\delta}\RR)}.
        \label{eqn:err}
    \end{multline}
    \label{thm:proj-gd}
\end{theorem}

Before we present the proof of Theorem \ref{thm:proj-gd}, we have two remarks on the impact of long/short-term dependencies in the dynamical system and the required network width.

\begin{remark}[Exploding gradient problem in the kernel regime]
     The terms $\Lambda_T$ and $\gamma_T$ that appear in Theorem \ref{thm:proj-gd} have a polynomial dependency in $\mu_T:=\mu_{0,T}(\sigma_1\alpha_m)\mu_{1,T}(\sigma_1\alpha_m)$, where $\alpha_m=\alpha+\frac{\rho_w}{\sqrt{m}}$ is an upper bound on $\max_{s=1,2,\ldots}\|\W(s)\|_{\infty,\infty}$. As such, $1/\sigma_1$ acts as a cutoff point for $\alpha_m$, which determines the memory in the system. For systems with long-term memory that require $\alpha_m \geq 1/\sigma_1$, Theorem \ref{thm:proj-gd} indicates that both $m$ and $\tau$ grow at a rate $\exp(\Omega(T))$ to guarantee a given target error. This pathological case is also characterized by exponentially growing magnitude for $\del_\Theta\er(\Theta)$, leading to the name \textit{exploding gradient problem}. In the benign case, i.e., for systems with short-term memory such that $\alpha_m < 1/\sigma_1$, the dependency of $m$ and $\tau$ to the sequence length $T$ is only polynomial since $\mu_T=\cO(1)$. 
     
     For systems with only a short-term memory, characterized by $\alpha +\frac{\nb_w}{\sqrt{m}} < 1/\sigma_1$, (i) max-norm projection with an accurate prior knowledge to select a projection radius $\rho_w\approx \nb_w$, and (ii) sufficiently large $m$ would yield $\alpha_m < 1/\sigma_1$, avoiding the exploding gradient problem. Several methods, including the celebrated long-short-term memory (LSTM), are introduced to address this issue for systems with long-term memory in practice \cite{graves2012long, goodfellow2016deep, hochreiter1997long}.
    \label{rem:lstm}
\end{remark}

\begin{remark}[Dependency of the $m$ on $n$]
     We observe in Theorem \ref{thm:proj-gd} that $m$ has only a logarithmic dependency on $n$, the number of training samples, to guarantee a given target error, unlike the existing works that require a polynomial dependency with massive exponents \cite{allen2019convergence}. The complexity of the problem in our analysis is measured in terms of $\|\bm{\nb}\|_2$, which bounds the RKHS norm of $F_t^\star,~t=1,2,\ldots,T$, and thus (i) the Lipschitz constant of the functions $F_t^\star$ in the RKHS $\cH_{\kappa_t}$, and (ii) how far $\Theta(s),s=1,2,\ldots$ should move away from $\Theta(0)$.
\end{remark}


\begin{proof}[Proof of Theorem \ref{thm:proj-gd}]
    For $(F_t^\star(\X))_{t\in\bZ_+}\in\iwl$, let $\bar{\Phi}$ be as defined in \eqref{eqn:projection} (see also Remark \ref{rem:transport}). By using this as the point of attraction, we use the following Lyapunov function:
\begin{equation*}
    \cL(\Phi)=\sum_{i=1}\|\Phi-\bar{\Phi}\|_2^2.
\end{equation*}
Since $\rho \succeq \bm{\nb}$, it is easy to see that $\bar{\Phi}\in\Omega_{\bm{\rho}}$. By the non-expansivity of projection onto a convex set \cite{brezis2011functional}, and the Pythagorean theorem, for any $s\in\bN$, we have 
\begin{align}
    \nonumber\cL\LL(\Phi(s+1)\RR)&\leq \cL\LL(\tilde{\Phi}(s)\RR)\\
    &=\cL\LL(\Phi(s)\RR)+2\eta\del_{\Phi}^\top\er(\Phi(s))\LL(\bar{\Phi}-\Phi(s)\RR)+\eta^2\|\del_{\Phi}\er(\Phi(s))\|_2^2,
    \label{eqn:drift}
\end{align}
where $\tilde{\Phi}(s) = \Phi(s)-\eta\del_\Phi\er(\Phi(s))$ is the pre-projection iterate. 

{First, we will show that $\del_\Phi^\top\er(\Phi(s))(\bar{\Phi}-\Phi(s)) = -2\er(\Phi(s)) + \mathcal{O}\LL(\frac{1}{\sqrt{m}}\RR)$, which yields the negative drift in \eqref{eqn:drift} with a controllable error term.} For any $j\in\{1,2,\ldots,n\}$,
\begin{equation*}
\del_{\Phi}^\top\er(\Phi(s))\LL(\bar{\Phi}-{\Phi}(s)\RR) = \frac{{2}}{n}\sum_{j, t}\big(F_t(\X^\uind{j};\Phi(s))-F_t^\star(\X^\uind{j})\big)\del_{\Phi}^\top F_t(\X^\uind{j};\Phi(s))\big(\bar{\Phi}-\Phi(s)\big).
\end{equation*}
{
For any $j\in\{1,2,\ldots,n\}$, a change-of-feature step yields the following decomposition:
\begin{equation*}
    \del_{\Phi}^\top F_t(\X^\uind{j};\Phi(s))\LL(\bar{\Phi}-\Phi(s)\RR)=\del_{\Phi}^\top F_t(\X^\uind{j};\Phi(0))\LL(\bar{\Phi}-\Phi(s)\RR)+\epsilon^\mathsf{CoF}_{j, t}(\Phi(s)),~j=1,\ldots,n,
\end{equation*}
where
\begin{equation*}
    \epsilon_{j,t}^\mathsf{CoF}(\Phi(s)):=\LL(\del_{\Phi} F_t(\X^\uind{j};\Phi(s))-\del_{\Phi} F_t(\X^\uind{j};\Phi(0))\RR)^\top\LL(\bar{\Phi}-\Phi(s)\RR),~j=1,\ldots,n.
\end{equation*}
Furthermore, since
\begin{equation*}
    \del_{\Phi}^\top F_t(\X;\Phi(0))\LL(\bar{\Phi}-\Phi(s)\RR)=\Lin(\X;\bar{\Phi})-\Lin(\X;\Phi(s))~\mbox{for any}~\X\in\bR^{d\times T},
\end{equation*}
we have the following error decomposition:
\begin{equation}
    \del_{\Phi}^\top F_t(\X;\Phi(0))\LL(\bar{\Phi}-\Phi(s)\RR) = F_t^\star(\X^\uind{j}) - F_t(\X^\uind{j};\Phi(s)) + \epsilon_{j,t}^\mathsf{Lin}(\Phi(s)),
\end{equation}
where the linearization error is
\begin{align*}
    \epsilon_{j,t}^\mathsf{Lin}(\Phi(s)) := \LL(F_t^\star(\X^\uind{j})-F_t^\mathsf{Lin}(\X^\uind{j};\bar{\Phi})\RR)+\LL(F_t(\X^\uind{j};\Phi(s))-F_t^\mathsf{Lin}(\X^\uind{j};\Phi(s)\RR).
\end{align*}
First, we provide an upper bound for $|\epsilon_{j,t}^\mathsf{Lin}(\Phi(s))|$.

As an important result that stems from the max-norm projection, we have a compact and convex parameter space $\Omega_{\bm{\rho}}$ such that $\Phi(s),\bar{\Phi}\in\Omega_{\bm{\rho}}$ for any $s\in\{0,1,\ldots\}$, which implies the applicability of Lemma \ref{lemma:lipschitz}-\ref{lemma:smoothness} with $\alpha_m=\alpha + \frac{\rho_w}{\sqrt{m}}$. 
By Prop. \ref{prop:lin-ntrf}-\ref{prop:lin-ntk}, 
for all $j=1,\ldots,n$,
\begin{align*}
    |\Lin(\X^\uind{j};\Phi(s))-F_t(\X^\uind{j};\Phi(s))| &\leq \frac{2}{\sqrt{m}}\Big(\Lambda_T^2\sigma_2+\gamma_T\sigma_1\Big)(\rho_w^2+\rho_u^2)+{\frac{2L_T}{\sqrt{m}}\rho_c\sqrt{\rho_w^2+\rho_u^2}}\\
    &=: \mathsf{err}_{1,app},
\end{align*}
and, for $\mu_0:=\mu_{0,T}(\alpha_m\sigma_1)$,
\begin{align*}
    |F^\star_t(\X^\uind{j})-\Lin(\X^\uind{j};\bar{\Phi})|&\leq 2\|\bm{\nb}\|_2\Big(\sigma_1\mu_{0}(1+\sigma_0)+\sigma_0\Big)\sqrt{\frac{\log(2nT/\delta)}{m}}\\
    &=:\mathsf{err}_{1,lin},
\end{align*}
simultaneously, with probability at least $1-\delta$ over the random initialization.}
{Therefore,
\begin{align}\max_{\substack{j=1,\ldots,n\\t=1,\ldots,T}}|\epsilon_{j,t}^\mathsf{Lin}(\Phi(s))| &\leq \mathsf{err}_{1,app}+\mathsf{err}_{1,lin} =:  \frac{\mathsf{err}_1}{\sqrt{m}},
\label{eqn:err1}
\end{align}
with $\mathsf{err}_{1,app}$ and $\mathsf{err}_{1,lin}$ defined above. Hence, we have
\begin{equation}\del_\Phi^\top\er(\Phi(s))(\bar{\Phi}-\Phi(s)) = -2\er(\Phi(s)) + \frac{2}{n}\sum_{j,t}(F_t(\X^\uind{j};\Phi(s))-F_t^\star(\X^\uind{j}))(\epsilon_{j,t}^\mathsf{Lin}+\epsilon_{j,t}^\mathsf{CoF}(\Phi(s))),
\label{eqn:negative-drift}
\end{equation}
where $\max_{j,t}|\epsilon_{j,t}^\mathsf{Lin}(\Phi(s))| \leq \mathsf{err}_1/\sqrt{m}$. In the following, we bound the change-of-measure error $|\epsilon_{j,t}^\mathsf{CoF}(\Phi(s))|$. Recall that
\begin{align*}
    \del_{\Theta_i}F_t(\X;\Phi(s))-\del_{\Theta_i}F_t(\X;\Phi(0)) &= \frac{c_i(0)}{\sqrt{m}}\LL(\del_{\Theta_i}H_t^\uind{i}(\X;\Theta(s))-\del_{\Theta_i}H_t^\uind{i}(\X;\Theta(0))\RR)\\
    &+ \frac{c_i-c_i(0)}{\sqrt{m}}{\del_{\Theta_i}H_t^\uind{i}(\X;\Theta(s))},
\end{align*}
and $\del_{c_i}F_t(\X;\Phi(s))=\frac{1}{\sqrt{m}}H_t^\uind{i}(\X;\Theta(s))$. Then, by using the fact that $\Theta_i\mapsto H_t^\uind{i}(\X;\Theta)$ is $L_t$-Lipschitz and {$\beta_t$-smooth} for $\Phi\in\Omega_{\bm{\rho}}$ by Proposition \ref{lemma:lipschitz}, we have
\begin{align*}
    \big\|\del_{\Theta_i} F_t(\X^\uind{j};\Phi(s))-\del_{\Theta_i} F_t(\X^\uind{j};\Phi(0))\big\|_2 & \leq \frac{\beta_t\|\Theta_i(s)-\Theta_i(0)\|_2}{\sqrt{m}} + \frac{L_t|c_i(s)-c_i(0)|}{\sqrt{m}},\\
    |\del_{c_i}F_t(\X^\uind{j};\Phi(s))-\del_{c_i}F_t(\X^\uind{j};\Phi(0))| &\leq \frac{L_t}{\sqrt{m}}\|\Theta_i(s)-\Theta_i(0)\|_2,~j=1,2,\ldots,n,
\end{align*}
for all $i=\{1,2,\ldots,m\}$. Since $\Phi(s)\in\Omega_{\bm{\rho}}$ for all $s\in\bZ_+$ due to max-norm projection, we have $\|\Theta_i(s)-\Theta_i(0)\|_2^2\leq \frac{\rho_u^2+\rho_w^2}{\sqrt{m}}$ and $|c_i(s)-c_i(0)|\leq \frac{\rho_c}{\sqrt{m}}$ for all $i\in\{1,\ldots,m\}$. Thus,
\begin{equation*}
    \|\del_{\Phi_i}F_t(\X^\uind{j};\Phi(s))-\del_{\Phi_i}F_t(\X^\uind{j};\Phi(0))\|_2 \leq \frac{1}{m}\LL((\beta_t+L_t)\sqrt{\rho_w^2+\rho_u^2} + L_t\rho_c\RR),~i=1,\ldots,m.
\end{equation*}
Also, we have $|\bar{c}_i-c_i(0)|\leq \frac{\nb_c}{\sqrt{m}}$ and $\|\bar{\Theta}_i-\Theta_i(0)\|_2\leq \frac{\sqrt{\nb_w^2+\nb_u^2}}{\sqrt{m}}$ for all $i=1,2,\ldots,m$. Thus, 
\begin{align*}
\|\bar{\Phi}_i-\Phi_i(0)\|_2\leq \frac{\|\bm{\nb}\|_2}{\sqrt{m}}~\mbox{and}~\|{\Phi}_i(0)-\Phi_i(s)\|_2\leq \frac{\|\bm{\rho}\|_2}{\sqrt{m}}
,~i=1,2,\ldots,m.
\end{align*}
By using triangle inequality and Cauchy-Schwarz inequality,
\begin{align}
    \nonumber \LL|\epsilon_{j,t}^\mathsf{CoF}(\Phi(s))\RR| &=\LL\|\sum_{i=1}^m\LL(\del_{\Phi_i} F_t(\X^\uind{j};\Theta(s))-\del_{\Phi_i} F_t(\X^\uind{j};\Theta(0))\RR)^\top\LL(\Phi_i(s)-\bar{\Phi}_i\RR)\RR\|_2\\
    \nonumber &\leq \sum_{i=1}^m\|\del_{\Phi_i} F_t(\X^\uind{j};\Theta(s))-\del_{\Phi_i} F_t(\X^\uind{j};\Theta(0))\|_2\cdot\|\Phi_i(s)-\bar{\Phi}_i\|_2\\
    &\leq ((\beta_t+L_t)\sqrt{\rho_w^2+\rho_u^2}+L_t\rho_c)\cdot\frac{\LL(\|\bm{\nb}\|_2+\|\rho\|_2\RR)}{\sqrt{m}}=:\frac{1}{\sqrt{m}}\mathsf{err}_2,
    \label{eqn:err2}
\end{align}
for any $j=1,\ldots,n$ and $t = 1,\ldots, T$. This implies that \begin{align}
\nonumber\max_{\substack{j=1,\ldots,n\\t=1,\ldots,T}}~{|}\epsilon_{j,t}^\mathsf{CoF}(\Phi(s)){|}\leq \frac{1}{\sqrt{m}}\mathsf{err}_2&,
\end{align}
since $\beta_t\leq \beta_T$ {(see \eqref{eqn:ht-smoothness} for the monotonicity of $t\mapsto\beta_t$)}. Hence, by using the bounds on $F_t(\cdot,\Theta(s))$ and $F_t^\star(\cdot)$ given in Prop. \ref{prop:norm-bounds},
\begin{equation}
    \del_\Phi^\top\er(\Phi(s))\LL(\bar{\Phi}-\Phi(s)\RR)\leq -2\er(\Phi(s))+
    \frac{2{T}}{\sqrt{m}}\varphi_T^\mathsf{max}(\bm{\rho},\bm{\nb})\LL(\mathsf{err}_1+\mathsf{err}_2\RR),
    \label{eqn:neg-drift}
\end{equation}
where $$\varphi_T^\mathsf{max}(\bm{\rho},\bm{\nb}) := F_T^\mathsf{max}(\bm{\rho})+y_T^\mathsf{max}(\bm{\nb}) \geq \max_{j=1,2,\ldots,n}|F_t(\X^\uind{j};\Theta(s))-F_t^\star(\X^\uind{j})|,$$ for all $t\in\{1,\ldots,T\}$ and {$s\in\bN$}. Note that the regularization to control the magnitude of $\Phi(s)-\Phi(0)$ is essential to establish the linearization results, as well as the norm bounds, both used above. Here, we establish this control directly via max-norm projection. }

The remaining part of the proof is to bound $\|\del_\Theta\er(\Theta(s))\|_2^2$. From Proposition \ref{prop:gradient},
\begin{align}
    \nonumber \|\del_{\Phi_i}F_t(\X^\uind{j};\Phi(s))\|_2&\leq \sqrt{\frac{c_i^2(s)}{m}\|\del_{\Theta_i}H_t^\uind{i}(\X^\uind{j};\Theta(s)\|_2^2 + \frac{1}{m}\sigma_0^2}\\
    &\leq \frac{1}{\sqrt{m}}\sqrt{\LL(1+\frac{\rho_c}{\sqrt{m}}\RR)^2L_t^2 + \sigma_0^2}\leq \frac{\LL(1+\frac{\rho_c}{\sqrt{m}}\RR)L_t + \sigma_0}{\sqrt{m}},\label{eqn:norm-bound-F}
\end{align}
for $i\in\{1,\ldots,m\},~j\in\{1,\ldots,n\}$, since $|c_i(s)|\leq |c_i(s)-c_i(0)|+|c_i(0)|\leq \frac{\rho_c}{\sqrt{m}}+1$. Therefore,
\begin{align}
    \nonumber \|\del_{\Phi_i}\er(\Phi(s))\|_2&\leq\frac{2}{n}\sum_{j=1}^n\sum_{t=1}^T|F_t(\X^\uind{j};\Phi(s))-F_t^\star(\X^\uind{j})|\cdot\|\del_{\Phi_i}F_t(\X^\uind{j};\Phi(s))\|_2\\
    &\leq \frac{2}{n}\sum_{j=1}^n\sum_{t=1}^T\bd\frac{(1+\frac{\rho_c}{\sqrt{m}})L_T+\rho_0}{\sqrt{m}}=2T\bd(\rho,\bm{\nb})\frac{(1+\frac{\rho_c}{\sqrt{m}})L_T+\sigma_0}{\sqrt{m}},
    \label{eqn:variability}
\end{align}
since $L_t\leq L_T$ for any $t \leq T$. Therefore, we have
\begin{equation*}
    \sup_{s\in\bZ_+}\|\del_{\Phi}\er(\Phi(s))\|_2^2\leq 4T^2\LL(\bd(\bm{\rho},\bm{\nb})\LL(1+\frac{\rho_c}{\sqrt{m}}\RR)L_T + \sigma_0\RR)^2 =: \mathsf{err}_3.
\end{equation*}
Substituting the bounds in \eqref{eqn:neg-drift} and \eqref{eqn:variability} into the Lyapunov drift bound in \eqref{eqn:drift}, we obtain
\begin{align}
    \nonumber \Delta(\Phi(s))&=\cL(\Phi(s+1))-\cL(\Phi(s)),\\
    &\leq -{4}\eta\er(\Phi(s))+4\eta T\bd(\bm{\rho},\bm{\nb})\frac{\mathsf{err}_1+\mathsf{err}_2}{\sqrt{m}}+\eta^2\mathsf{err}_3,
    \label{eqn:ly-drift}
\end{align}
for every $s=0,1,\ldots,\tau-1$. Then, by noting that $\sum_{s<\tau}\Delta(\Phi(s))=\cL(\Phi(\tau))-\cL(\Phi(0))$,
\begin{equation*}
   \frac{\cL(\Phi(\tau))-\cL(\Phi(0))}{{4}\eta\tau}\leq -\frac{1}{\tau}\sum_{s<\tau}\er(\Phi(s))+ T\bd(\bm{\rho},\bm{\nb})\frac{\mathsf{err}_1+\mathsf{err}_2}{\sqrt{m}}+\frac{\eta}{{4}} \cdot \mathsf{err}_3.
\end{equation*}
By re-arranging terms, we obtain
\begin{equation*}
    \frac{1}{\tau}\sum_{s<\tau}\er(\Phi(s))\leq \frac{\cL(\Phi(0))}{{4}\eta\tau}  +T\bd(\bm{\rho},\bm{\nb})\frac{\mathsf{err}_1+\mathsf{err}_2}{\sqrt{m}}+\frac{\eta}{{4}}\cdot\mathsf{err}_3.
\end{equation*}
The result follows from substituting $\cL(\Phi(0))\leq \|\bm{\nb}\|_2^2$ and the step-size choice $\eta = \frac{1}{T\sqrt{\tau}}$ into the above inequality. Consequently, we get
\begin{equation}
    \frac{1}{\tau}\sum_{s<\tau}\er(\Phi(s))\leq \frac{T}{4\sqrt{\tau}}\|\bm{\nb}\|_2^2+T\bd(\bm{\rho},\bm{\nb})\frac{\mathsf{err}_1+\mathsf{err}_2}{ \sqrt{m}}+\frac{\mathsf{err}_3}{{4}T\sqrt{\tau}}.
    \label{eqn:proj-gd-error}
\end{equation}
The final result is obtained by noting that $$\min_{0\leq s<\tau}\er(\Phi(s))\leq \frac{1}{\tau}\sum_{s=0}^{\tau-1}\er(\Phi(s)),$$ and using the upper bound \eqref{eqn:proj-gd-error}.
\end{proof}

\begin{corollary}[Average-iterate convergence of projected gradient descent for RNNs]
    In the same setting as Theorem \ref{thm:proj-gd}, max-norm-projected gradient descent with projection radii $\bm{\rho} \succeq \bm{\nb}$ and step-size $\eta = \frac{1}{T\sqrt{\tau}}$ after $\tau>1$ iterations yields
    \begin{multline*}
        \er\LL(\frac{1}{\tau}\sum_{s<\tau}\Phi(s)\RR)\leq \frac{2T}{\sqrt{\tau}}\LL(\|\bm{\nb}\|_2^2+\Big(\bd(\bm{\rho},\bm{\nb})\big(1+\frac{\rho_c}{\sqrt{m}}\big)L_T + \sigma_0\Big)^2\RR)+\frac{4T\bd(\bm{\rho},\bm{\nb})}{\sqrt{m}}\mathsf{err}\\+\frac{24T}{m}\LL(\LL(\Lambda_t^2\sigma_2+\gamma_t\sigma_1\RR)(\rho_w^2+\rho_u^2)+L_T\rho_c\sqrt{\rho_w^2+\rho_u^2}\RR)^2,
    \end{multline*}
    with probability at least $1-\delta$ over the random initialization, where $\mathsf{err}$ is defined in \eqref{eqn:err}.
    \label{cor:avg-iterate-1}
\end{corollary}
\begin{proof}
    Let $$\er^\mathsf{Lin}(\Phi)=\frac{1}{n}\sum_{j=1}^n\sum_{t=1}^T\LL(\Lin(\X^\uind{j};\Phi)-F_t^\star(\X^\uind{j})\RR)^2,$$
    be the empirical risk under the linearized model $\Lin$. Then, it is straightforward to see that
    \begin{align*}
        \frac{1}{\tau}\sum_{s=0}^{\tau-1}\er^\mathsf{Lin}(\Phi(s))&\leq \frac{2}{\tau}\sum_{s=0}^{\tau-1}\er(\Phi(s))+\frac{2}{n\tau}\sum_{j=1}^n\sum_{s<\tau}{\sum_{t=1}^T}\LL(F_t(\X^\uind{j};\Phi(s))-\Lin(\X^\uind{j};\Phi(s))\RR)^2\\
        &\leq \frac{2}{\tau}\sum_{s=0}^{\tau-1}\er(\Phi(s))+\frac{8T}{m}\LL(\LL(\Lambda_T^2\sigma_2+\gamma_T\sigma_1\RR)(\rho_w^2+\rho_u^2)+L_T\rho_c\sqrt{\rho_w^2+\rho_u^2}\RR)^2,
    \end{align*}
    where the first line above follows from the inequality $(x+y)^2\leq 2x^2+2y^2$, and the second line is a direct consequence of the projection step and the linearization error in Prop. \ref{prop:lin-ntk}. Using an identical argument, we also have
    \begin{align*}
        \er(\Phi)\leq 2\er^\mathsf{Lin}(\Phi)+\frac{8T}{m}\LL(\LL(\Lambda_T^2\sigma_2+\gamma_T\sigma_1\RR)(\rho_w^2+\rho_u^2)+L_T\rho_c\sqrt{\rho_w^2+\rho_u^2}\RR)^2.
    \end{align*}
    From Jensen's inequality, {since $\Phi\mapsto\er^\mathsf{Lin}(\Phi)$ is convex as it takes the form of a quadratic (loss) function of a linear function (predictor $F_t^\Sf{Lin})$ of $\Phi$}, we have $\er^\mathsf{Lin}\LL(\frac{1}{\tau}\sum_{s<\tau}\Phi(s)\RR)\leq \frac{1}{\tau}\sum_{s<\tau}\er^\mathsf{Lin}(\Phi(s)).$ From these three results, we obtain
    \begin{align}
        \nonumber \er\Big(\frac{1}{\tau}\sum_{s<\tau}\Phi(s)\Big)&\leq 2\er^\mathsf{Lin}\Big(\frac{1}{\tau}\sum_{s<\tau}\Phi(s)\Big)+ \frac{8T}{m}\LL(\LL(\Lambda_T^2\sigma_2+\gamma_T\sigma_1\RR)(\rho_w^2+\rho_u^2)+L_T\rho_c\sqrt{\rho_w^2+\rho_u^2}\RR)^2\\
        \nonumber &\leq \frac{2}{\tau}\sum_{s<\tau}\er^\mathsf{Lin}\LL(\Phi(s)\RR)+ \frac{8T}{m}\LL(\LL(\Lambda_T^2\sigma_2+\gamma_T\sigma_1\RR)(\rho_w^2+\rho_u^2)+L_T\rho_c\sqrt{\rho_w^2+\rho_u^2}\RR)^2\\
        &\leq \frac{4}{\tau}\sum_{s<\tau}\er\LL(\Phi(s)\RR)+ \frac{24T}{m}\LL(\LL(\Lambda_T^2\sigma_2+\gamma_T\sigma_1\RR)(\rho_w^2+\rho_u^2)+L_T\rho_c\sqrt{\rho_w^2+\rho_u^2}\RR)^2.
        \label{eqn:linearization-argument}
    \end{align}
    Finally, recall that the sequence $\{\Phi(s):s=0,1,\ldots\}$ is obtained by max-norm-projected gradient descent, and substituting \eqref{eqn:proj-gd-error} into the last inequality concludes the proof.
\end{proof}

\begin{corollary}[Convergence of projected stochastic gradient descent]
    In the same setting as Theorem \ref{thm:proj-gd}, max-norm-projected stochastic gradient descent with projection radii $\bm{\rho}\succeq\bm{\nb}$ and step-size $\eta = \frac{1}{T\sqrt{\tau}}$ after $\tau>1$ iterations yields
    \begin{multline*}
        \bE\Big[\er\Big(\frac{1}{\tau}\sum_{s<\tau}\Phi(s)\Big)\Big|\Phi(0)\Big]\leq \frac{2T}{\sqrt{\tau}}\LL(\|\bm{\nb}\|_2^2+\Big(\bd(\bm{\rho},\bm{\nb})\big(1+\frac{\rho_c}{\sqrt{m}}\big)L_T + \sigma_0\Big)^2\RR)\\+\frac{4T\bd(\rho,\bm{\nb})}{\sqrt{m}}\mathsf{err}+\frac{24T}{m}\LL(\LL(\Lambda_t^2\sigma_2+\gamma_t\sigma_1\RR)(\rho_w^2+\rho_u^2)+L_T\rho_c\sqrt{\rho_w^2+\rho_u^2}\RR)^2,
    \end{multline*}
    with probability at least $1-\delta$ over the random initialization, where $\mathsf{err}$ is defined in \eqref{eqn:err}.
    \label{cor:sgd-avg-iterate-1}
\end{corollary}
\begin{proof}
    Let $\calF_s:=\varsigma(\Phi(0),J_0,\ldots,J_{s-1})$ for $s\in\bZ_+$ with $\calF_0:=\varsigma(\Phi(0))$. First, note that $\max_{i=1,\ldots,m}\|G_{\Phi_i}(\Phi(s),s)\|_2\leq \frac{\bd(\rho,\nb)}{\sqrt{m}}\Big((1+\rho_c/\sqrt{m})L_T+\sigma_0\Big)$ surely, which implies that $$\|\bm{G}_\Phi(\Phi(s),s)\|_2^2\leq 4T^2\LL[\LL(1+\frac{\rho_c}{\sqrt{m}}\RR)L_T+\sigma_0\RR]^2=\mathsf{err}_3,~\mbox{a.s.}$$ Also, $\bE[\bm{G}_\Phi(\Phi(s),s)|\calF_s]=\del_\Phi\er(\Phi(s))$ as $\Phi(s)$ is $\calF_s$-measurable $J_s\sim\mathsf{Unif}(\{1,\ldots,n\})$. Thus,
    \begin{align*}
        \bE[\Delta(\Phi(s))|\calF_s] &\leq \del_\Phi^\top\er(\Phi(s))\LL(\bar{\Phi}-\Phi(s)\RR)+\eta^2\mathsf{err}_3\\
        &\leq -{4}\eta\er(\Phi(s))+4\eta T\bd(\rho,\bm{\nb})\frac{\mathsf{err}_1+\mathsf{err}_2}{\sqrt{m}}+\eta^2\mathsf{err}_3,
    \end{align*} 
    on a set $A_0\in\mathcal{F}_0$ such that $\bP(A_0)\geq 1-\delta$, following identical steps as in \eqref{eqn:ly-drift} with identical $\mathsf{err}_{i},~i=1,2,3$. Using the tower property, for all $s\in\bN$ simultaneously, in $A_0\in\mathcal{F}_0$, we have
    $$
    \bE[\Delta(\Phi(s))|\calF_0] \leq -{4}\eta\bE[\er(\Phi(s))|\calF_0]+4\eta T\bd(\bm{\rho},\bm{\nb})\frac{\mathsf{err}_1+\mathsf{err}_2}{\sqrt{m}}+\eta^2\mathsf{err}_3.
    $$ Following similar steps after \eqref{eqn:ly-drift} in Theorem \ref{thm:proj-gd}, we obtain the same upper bound as Theorem \ref{thm:proj-gd} for $\frac{1}{\tau}\sum_{s<\tau}\bE[\er(\Phi(s))|\calF_0]$. Note that \eqref{eqn:linearization-argument} holds surely, thus taking $\bE[\cdot|\calF_0]$ on both sides and using the upper bound for $\frac{1}{\tau}\sum_{s<\tau}\bE[\er(\Phi(s))|\calF_0]$, we conclude the proof.
\end{proof}
\subsection{Convergence of Gradient Descent for RNNs (Projection-Free)}
In the following, we prove finite-time and finite-width convergence bounds for projection-free gradient descent. 
\begin{theorem}
    For any given target error $\epsilon > 0$, let $\Lambda_T,\gamma_T,L_T,\beta_T$ as in Lemmas \ref{lemma:lipschitz}-\ref{lemma:smoothness} with $\alpha_m = \alpha + \sqrt{2\epsilon}$. Also, let $\bm{\lambda} = (\lambda_w,\lambda_u,\lambda_c)^\top$ with $\lambda_w=\lambda_u=\frac{16\|\bm{\nb}\|_2^2\sqrt{T}L_T}{\sqrt{\epsilon}},$ $\lambda_c=\frac{2\sigma_0\sqrt{T}}{\sqrt{\epsilon}}$. For any $\delta \in (0,1)$, if $$m\geq \frac{T}{\epsilon}\LL(\mathsf{err}_1\big|_{\scriptsize \hskip -4.5pt \begin{array}{l}\bm{\rho} = \bm{\lambda}\\\mu_0=\mu_{0,T}(\alpha_\epsilon\sigma_1)\end{array}}+\mathsf{err}_2\big|_{\bm{\rho}=\bm{\lambda}}\RR) + \frac{\lambda_w^2}{2\epsilon} + \lambda_c^2,$$ where $\Sf{err}_1,\Sf{err}_2$ are defined in \eqref{eqn:err1} and \eqref{eqn:err2}, respectively, then gradient descent with step-size $
        \eta \leq \frac{\epsilon}{2T^2\LL(2\varphi_T^\mathsf{max}(\bm{\lambda},\bm{\nb})L_T + \sigma_0\RR)^2}$ after $\tau \geq \tau_0(\epsilon)=\frac{\|\bm{\nb}\|_2^2}{(\sqrt{2}+1)^2\eta\epsilon}$ iterations yields
    \begin{equation*}
        \min_{0\leq s<\tau}\er(\Phi(s))\leq 2\epsilon,
    \end{equation*}
    with probability at least $1-\delta$ over the random initialization.
    \label{thm:gd}
\end{theorem}

\begin{proof}
    Given any $\epsilon > 0$, we define the exit time under gradient descent as 
    \begin{align*}
        \tau_{\exit}=\inf\Bigg\{s\geq 0:\exists i\in[m]~\mbox{s.t.}~|\Wii(s)-\Wii(0)|_2 > \frac{\lambda_w}{\sqrt{m}}~&\mbox{or~}\|U_i(s)-U_i(0)\|_2 > \frac{\lambda_u}{\sqrt{m}}\\&\mbox{or}~|c_i(s)-c_i(0)|\geq \frac{\lambda_c}{\sqrt{m}}\Bigg\}.
    \end{align*}
    The intuition behind the proof is that (i) the drift bound \eqref{eqn:drift}, and (ii) the local Lipschitz continuity and smoothness results in Lemmas \ref{lemma:lipschitz}-\ref{lemma:smoothness} simultaneously hold for $s<\tau_\exit$, and with a particular choice of $\tau_0$, one can show that $\tau_0<\tau_\exit$ and also $\min_{s<\tau_0}\er(\Theta(s))\leq 2\epsilon$. 
    
    For $s<\tau_\exit$, the linearization results as well as norm bounds hold. {Thus, from \eqref{eqn:negative-drift},}
    \begin{equation*}
        \Delta(\Phi(s)) \leq -4\eta\er(\Phi(s))+4\eta{\sqrt{T}}\sqrt{\er(\Phi(s))}\frac{\mathsf{err}_1+\mathsf{err}_2}{\sqrt{m}}+\eta^2\mathsf{err}_3,
    \end{equation*}
    where $\Delta(\Phi(s))=\cL(\Phi(s+1))-\cL(\Phi(s))$, and $\mathsf{err}_i,i=1,2,3$ are defined in the proof of Theorem \ref{thm:proj-gd}. The choices of $\eta$ and $m$ in Theorem \ref{thm:gd} ensure that
    \begin{align*}
        \eta \cdot \mathsf{err}_3\big|_{\bm{\rho}=\bm{\lambda}}&\leq {\frac{\epsilon}{2}},\\
        \frac{1}{\sqrt{m}}\Big(\mathsf{err}_1\big|_{\scriptsize \hskip -4.5pt \begin{array}{l}
        \bm{\rho}=\bm{\lambda}\\\mu_0=\mu_{0,T}(\alpha_\epsilon\sigma_1)\end{array}}+\mathsf{err}_2\big|_{\bm{\rho}=\bm{\lambda}}\Big)&\leq \sqrt{\frac{\epsilon}{T}},\\
        \frac{\lambda_w}{\sqrt{m}}\leq \sqrt{2\epsilon}&~\mbox{and}~\frac{\lambda_c}{\sqrt{m}}\leq 1.
    \end{align*}
    Then, we have
    \begin{align*}
        \Delta(\Phi(s))&\leq -4\eta\er(\Phi(s)) + 4\eta\sqrt{\epsilon\cdot\er(\Phi(s))} + \frac{\eta\epsilon}{2} \\&= -4\eta\LL(\sqrt{\er(\Phi(s))}-\frac{\sqrt{\epsilon}}{2}\RR)^2+\frac{3\eta\epsilon}{2},
    \end{align*}
    which implies that $$\cL\LL(\Phi(\tau_\exit)\RR)-\cL(\Phi(0))\leq -4\eta\sum_{s<\tau_\exit}\LL(\sqrt{\er(\Phi(s))}-\frac{\sqrt{\epsilon}}{2}\RR)^2+\frac{3\eta\epsilon\tau_\exit}{2}.$$ Then, for any $\tau \leq \tau_\exit$, we have
    \begin{equation}
        \sum_{s<\tau}\LL(\sqrt{\er(\Phi(s))}-\frac{\sqrt{\epsilon}}{2}\RR)^2\leq\frac{\cL(\Phi(0))}{4\eta}+\frac{\epsilon\tau}{2}\leq \frac{\|\bm{\nb}\|_2^2}{2\eta}+\frac{\epsilon\tau}{2}.
        \label{eqn:exit-drift}
    \end{equation} {As $\Big(\frac{1}{\tau}\sum_{s<\tau}\sqrt{\er(\Phi(s))}-\frac{\sqrt{\epsilon}}{2}\Big)^2\leq \frac{1}{\tau}\sum_{s<\tau}\Big(\sqrt{\er(\Phi(s))}-\frac{\sqrt{\epsilon}}{2}\Big)^2,~\tau\leq\tau_\exit$ by Jensen inequality,}
    $$\sum_{s<\tau_\exit}\sqrt{\er(\Phi(s))}-\frac{\tau_\exit\sqrt{\epsilon}}{2}\leq \frac{\|\bm{\nb\|}_2\sqrt{\tau_\exit}}{\sqrt{2\eta}}+\tau_\exit\sqrt{\frac{\epsilon}{2}}.$$ Let $\tau_0\in\bZ_+$ be such that
    \begin{equation}
    \frac{\|\bm{\nb}\|_2^2}{\eta}= \frac{\epsilon\tau_0(1+\sqrt{2})^2}{2},
    \label{eqn:tau-choice}
    \end{equation} 
    for $\eta$ defined above. In the following, we will prove that $\tau_0< \tau_\exit$ and $\min_{s<\tau_0}\er(\Phi(s))\leq 2\epsilon$. Suppose to the contrary that $\tau_0>\tau_\exit$. Then, from the above inequality,
\begin{equation}\sum_{s<\tau_\exit}\sqrt{\er(\Phi(s))}\leq \sqrt{\epsilon}\tau_0\LL(\frac{1}{2}+\frac{1}{\sqrt{2}}\RR)+\frac{\|\bm{\nb}\|_2\sqrt{\tau_0}}{\sqrt{2\eta}} = (\sqrt{2}+1)\sqrt{\epsilon}\tau_0.
\label{eqn:sqrt-risk}
\end{equation}
    Now, note that, for any $i\in\{1,\ldots,m\}$ and $k \leq \tau_\exit$,
    \begin{align*}
        |c_i(k)-c_i(0)| &\leq \sum_{s<\tau_\exit}|c_i(s+1)-c_i(s)|=\eta\sum_{s<\tau_\exit}|\del_{c_i}\er(\Phi(s))|\\
        &=\sum_{s<\tau_\exit}\eta\LL|\frac{2}{n}\sum_{j=1}^n\sum_{t=1}^T\frac{1}{\sqrt{m}}H_t^\uind{i}(\X^\uind{j};\Phi(s))\LL(F_t(\X^\uind{j};\Phi(s))-F_t^\star(\X^\uind{j})\RR)\RR|\\
        &\leq \frac{2\eta}{n\sqrt{m}}\sum_{j=1}^n\sum_{t=1}^T|H_t^\uind{i}(\X^\uind{j};\Phi(s))|\cdot|F_t(\X^\uind{j};\Phi(s))-F_t^\star(\X^\uind{j})|\\
        &\leq \frac{2\eta\sqrt{T}\sigma_0}{\sqrt{m}}\sum_{s<\tau_\exit}\sqrt{\er(\Phi(s))} \leq \frac{2\eta\sqrt{T}\sigma_0}{\sqrt{m}}(\sqrt{2}+1)\sqrt{\epsilon}\tau_0,
    \end{align*}
    where the fourth inequality follows from Cauchy-Schwarz inequality, and the last inequality is due to \eqref{eqn:sqrt-risk}. From \eqref{eqn:tau-choice}, $|c_i(k)-c_i(0)| \leq \frac{2\sigma_0\sqrt{T}}{\sqrt{\epsilon}\cdot\sqrt{m}} = \frac{\lambda_c}{\sqrt{m}}$, which implies $|c_i(k)| \leq 1 + \frac{\lambda_c}{\sqrt{m}}\leq 2$ for $k=0,1,\ldots,\tau_\exit$ since $m \geq \lambda_c^2$. For any $i=1,\ldots,m$,    
    \begin{align*}
        |\Wii(\tau_\exit)-\Wii(0)|&\leq \sum_{s<\tau_\exit}|\Wii(s+1)-\Wii(s)|\\&\leq \frac{2\eta}{n}\sum_{s<\tau_\exit}\frac{|c_i(s)| L_T}{\sqrt{m}}\sum_{j=1}^n\sum_{t=1}^T|F_t(\X^\uind{j};\Phi(s))-F_t^\star(\X^\uind{j})|\\
        &\leq \frac{2\eta L_T{\sqrt{T}}}{\sqrt{m}}\sum_{s<\tau_\exit}\LL(1+\frac{\lambda_c}{\sqrt{m}}\RR)\sqrt{\er(\Phi(s))}\\
        &\leq \LL(1+\frac{\lambda_c}{\sqrt{m}}\RR)\frac{4\sqrt{2}{\sqrt{T}}\eta L_T\sqrt{\epsilon}\tau_0}{\sqrt{m}}\\&\leq\LL(1+\frac{\lambda_c}{\sqrt{m}}\RR)\frac{8\|\bm{\nb}\|_2^2\sqrt{T} L_T}{\sqrt{m}\sqrt{\epsilon}}
        =\frac{\lambda_w}{\sqrt{m}},
    \end{align*}
    where the second line holds since $|\del_{\Wii} F_t|\leq \|\del_{\Theta_i} F_t\|_2$ and from the bound in Prop. \ref{lemma:lipschitz}, and the last inequality holds since $\lambda_c/\sqrt{m}\leq 1$. The same argument is followed to establish that $$\|U_i(\tau_\exit)-U_i(0)\|_2\leq \frac{16\|\bm{\nb}\|_2^2L_T\sqrt{T}}{\sqrt{\epsilon m}}= \frac{\lambda_u}{\sqrt{m}}.$$ This yields a contradiction with the definition of $\tau_\exit$, thus we prove that $\tau_\exit > \tau_0$. Using this result in \eqref{eqn:exit-drift}, since $z\mapsto(\sqrt{z}-\frac{\sqrt{\epsilon}}{2})^2$ is convex, we have
    \begin{equation*}
        \frac{1}{\tau_0}\sum_{s<\tau_0}\er(\Phi(s))-\frac{\epsilon}{4}\leq \frac{\|\bm{\nb}\|_2^2}{2\eta\tau_0}+\frac{\epsilon}{2}\leq \LL(\frac{(\sqrt{2}+1)^2}{4}+\frac{1}{2}\RR)\epsilon,
    \end{equation*} from the choice of $\tau_0$ in \eqref{eqn:tau-choice}. Thus, we conclude that 
    \begin{equation}
        \min_{0\leq s < \tau_0}\er(\Phi(s))\leq \frac{1}{\tau_0}\sum_{s<\tau_0}\er(\Phi(s))\leq 2\epsilon.
        \label{eqn:early-stopping}
    \end{equation}
    Hence, for any $\tau \geq \tau_0$, we obtain $\min_{0\leq s < \tau}\er(\Phi(s))\leq \min_{0\leq s < \tau_0}\er(\Phi(s)\leq 2\epsilon.$
\end{proof}

\begin{remark}[Average-iterate convergence under early stopping]
    It can be seen from \eqref{eqn:early-stopping} that an average-iterate convergence bound for gradient descent can be established by following identical steps as in Corollary \ref{cor:avg-iterate-1} for projected gradient descent. However, the number of iterations should be $\tau_0$ for the average-iterate bound in this case unlike Corollary \ref{cor:avg-iterate-1}, which stems from the early-stopping analysis that we establish here, inspired by \cite{cayci2023sample, ji2019polylogarithmic}.
\end{remark}

\begin{remark}[Representational assumption.]
    In our results, we consider the class of dynamical systems that can be represented by an infinitely-large RNN, characterized as $\cF_{\bm{\nb}}$, akin to \cite{ji2019polylogarithmic, cayci2023sample}. An alternative representational assumption, based on the strict positive-definiteness of the NTK matrices (which is effectively an assumption on the data points $\X^\uind{j}$) is also widely used in the NTK literature \cite{du2019gradient, chizat2019lazy}. Unlike these works, we do not have any assumptions on the data points, or the NTK matrix. Instead, we assume realizability in the class of infinite-width RNNs for simplicity. The elimination of this realizability assumption would lead to an approximation error, which was characterized for the special case of feedforward neural networks by the universality results in \cite{ji2019neural}.
\end{remark}

\section{Conclusion and Future Work}
In this work, we present a finite-time convergence analysis of RNNs trained with gradient descent, and proved improved bounds on the network size in terms of $n,T$ and $\delta$. Our analysis demonstrated the significant impact of the long/short-term dependencies on the required network size and iteration complexity.

Our analysis heavily exploited the smoothness of the nonlinear activation function $\sigma$. The extension of our work to the case of nondifferentiable $\sigma$, such as ReLU networks, necessitates new proof techniques, and is an interesting future direction.

In practical applications, long-short-term memory (LSTM) \cite{hochreiter1997long} and gated recurring unit (GRU) \cite{cho2014learning} mechanisms are incorporated into RNNs to learn long-term dependencies without suffering from the exploding gradient problem, which appeared as $\exp(\Omega(T))$ dependencies for $m$ and $\tau$ in our work (see Remark \ref{rem:lstm}). The convergence analysis of RNNs that employ these mechanisms is an interesting future direction as well.

\appendix
\section*{Acknowledgement}
Atilla Eryilmaz was supported by NSF grants:  NSF AI Institute (AI-EDGE) 2112471, CNS-NeTS-2106679,  CNS-NeTS-2007231; and the ONR Grant N00014-19-1-2621.
\section{Proofs for Lipschitz Continuity and Smoothness Results}
\subsection{Proof of Lemma \ref{lemma:lipschitz}}\label{app:ht-lipschitz}
\begin{proof}[Proof of Lemma \ref{lemma:lipschitz}] Throughout the proof, we consider $i\in\{1,\ldots,m\},~t\in\{1,\ldots,T\}$, and denote $H_t^\uind{i}=H_t^\uind{i}(\X;\Theta)$, $\alpha_m\geq |\Wii|$, {$\mu_{0,t,m}:=\ms(\alpha_m\sigma_1)$} and {$\mu_{1,t,m}:=\mb(\alpha_m\sigma_1)$} for notational simplicity, where $\ms$ and $\mb$ are defined in \eqref{eqn:zeta}. Recall from Prop. \ref{prop:gradient} that 
    \begin{align*}
        |\del_{\Wii} H_t^{\uind{i}}|&=\Big|\sum_{k=0}^{t-1}W_{ii}^kH_{t-k-1}^{(i)}I_{t-1}^\uind{i}\ldots I_{t-k-1}^\uind{i}\Big|\leq \sum_{k<t}\alpha_m^k\sigma_0\sigma_1^{k+1}\leq \sigma_0\sigma_1{\mu_{0,t,m}},\\
        \|\del_{U_i}H_t^\uind{i}\|_2&\leq \sigma_1{\mu_{0,t,m}}.
    \end{align*}
    from triangle inequality since $|H_t^\uind{i}|\leq\|\sigma\|_\infty\leq \sigma_0$, $|I_t^\uind{i}|\leq\|\sigma'\|_\infty\leq \sigma_1$ and $|\Wii|\leq |\Wii(0)|+|\Wii-\Wii(0)|\leq \alpha_m$. Therefore, $\|\del_{\Theta_i}H_t^\uind{i}\|_2^2=|\del_{\Wii} H_t^{\uind{i}}|^2+\|\del_{U_i}H_t^\uind{i}\|_2^2\leq {\mu_{0,t,m}^2}\cdot(\sigma_0^2+1)\cdot\sigma_1^2$, which concludes the proof for the Lipschitz continuity of $\Theta_i\mapsto H_t^\uind{i}(\X;\Theta)$.

    In order to prove local $\beta_t$-smoothness of $\Theta_i\mapsto H_t^\uind{i}$, we will prove an upper bound $\beta_t<\infty$ on $\|\del_{\Theta_i}^2 H_t^\uind{i}\|_F$ where $\del_{\Theta_i}^2$ denotes the Hessian matrix. For $s\in\{1,\ldots,d\}$, we have \begin{align*}
        \del_{U_{is}}I_{i,k}&=\LL(\Wii\del_{U_{is}}H_k^\uind{i}+X_{k,s}\RR)\sigma''(\Wii H_{k}^\uind{i}+U_i^\top X_{k}),\\
        \del_{\Wii}I_{i,k}&=\LL(H_k^\uind{i}+\Wii\del_{\Wii}H_k^\uind{i}\RR)\sigma''(\Wii H_{k}^\uind{i}+U_i^\top X_{k}),
        \end{align*} 
    Then, $\del_{U_{ir},U_{is}}^2H_t^\uind{i}=\sum_{k<t}\Wii^kX_{t-k-1,r}\sum_{l=0}^k\del_{U_{is}}I_{i,t-l-1}\prod_{\substack{j=0\\j\neq l}}^kI_{i,j},$
    which further implies, \begin{align*}|\del_{U_{ir},U_{is}}^2H_t^\uind{i}|\leq \underbrace{\sigma_2\LL[{\alpha_m}\sigma_1{\mu_{0,t,m}}+1\RR]{\mu_{1,t,m}}}_{(\clubsuit)},~\mbox{for any}~r,s\in\{1,\ldots,d\},
    \end{align*}
    since ${\mu_{0,t,m}}$ is monotonically increasing in $t$ for any $z\in\bR_\gz$. Likewise, for any $r\in\{1,\ldots,d\}$,
    \begin{align*}
        |{\del^2_{\Wii,U_{ir}}}H_t^\uind{i}|&\leq \underbrace{\sigma_1^2{\mu_{0,t,m}^2}+\sigma_0\sigma_2{(1+\mu_{0,t,m}\sigma_1\alpha_m)\mu_{1,t,m}}}_{(\diamondsuit)},\\
        |{\del^2_{U_{ir},\Wii}}H_t^\uind{i}|&\leq\underbrace{ {\mu_{1,t,m}}\sigma_1^2+\sigma_0\sigma_2(1+{\alpha_m}\sigma_1{\mu_{0,t,m}}){\mu_{1,t,m}}}_{(\heartsuit)}.
    \end{align*}
    Finally, we have
        $|\del_{\Wii}^2H_t^\uind{i}| \leq \underbrace{\sigma_1^2{\mu_{1,t,m}}+\sigma_1^2\sigma_0{\mu_{0,t,m}^2}+\sigma_0\sigma_2[1+{\alpha_m}\sigma_1{\mu_{0,t,m}}]{\mu_{1,t,m}}}_{(\spadesuit)}.$
    From these,
    \begin{align*}
        \|{\del^2_{\Theta_i}}H_t^\uind{i}\|_F \leq d\cdot(\clubsuit)+\sqrt{d}\cdot\big[(\diamondsuit)+(\heartsuit)\big]+(\spadesuit)=:\beta_t,
    \end{align*}
    by using the inequality $\sqrt{a+b}\leq \sqrt{a}+\sqrt{b}$ for any $a,b\in\bR_{\geq 0}$. Thus, $\|\del_{\Theta_i}^2H_t^\uind{i}\|_F\leq\beta_t$ where 
    \begin{align}
        \beta_t := \sigma_1^2\mu_{0,t,m}^2(\sqrt{d}+d+\sigma_0) + \sqrt{d}\sigma_1^2\mu_{1,t,m}+\sigma_2(1+\alpha_m\sigma_1\mu_{0,t,m})\mu_{1,t,m}\big(2\sigma_0\sqrt{d}+d+\sigma_0\big).
        \label{eqn:ht-smoothness}
\end{align}
{Since $\mu_{0,t}$ and $\mu_{1,t}$ are both non-decreasing in $t\in\bN$, $t\mapsto\beta_t$ is non-decreasing.}
\end{proof}

\begin{proof}[Proof of Lemma \ref{lemma:smoothness}]
    First, note that
        $\|\del_{\Theta_i}\pact{t}(\Theta)\|_2^2\leq 2\LL(\sigma_0^2+1+\alpha_m^2\|\del_{\Theta_i}H_t^\uind{i}\|_2^2\RR),$
    which would yield $\|\del_{\Theta_i}\pact{t}(\Theta)\|_2\leq  \Lambda_t= \sqrt{2}(\sigma_0+1+\alpha_m L_t)$ from the inequalities $(a+b)^2\leq 2(a^2+b^2)$ and $\sqrt{a+b}\leq \sqrt{a}+\sqrt{b}$ for $a,b\in\bR_{\geq 0}$, and Lemma \ref{lemma:lipschitz}. This would conclude the proof of Lipschitz continuity of $\pact{t}$. For the smoothness, using the same inequalities, we have
    \begin{equation*}
        \|\del{^2}_{\Theta_i}\pact{t}\|_F^2 \leq 2\|\del_{\Theta_i}H_t^\uind{i}\|_2^2 + 2\Wii^2\|\del^2_{\Theta_i}H_t^\uind{i}\|_F^2.
    \end{equation*}
    Using Lemma \ref{lemma:lipschitz} in the above inequality, we conclude the proof with $\gamma_t = \sqrt{2}L_t+\sqrt{2}\alpha_m\beta_t.$
\end{proof}

\bibliographystyle{siamplain}
\bibliography{references}

\end{document}